\documentclass{article}

\usepackage{microtype}
\usepackage{graphicx}
\usepackage{subfigure}
\usepackage{booktabs} 

\usepackage{hyperref}

\usepackage[accepted]{icml2023}

\usepackage{amsmath}
\usepackage{amssymb}
\usepackage{mathtools}
\usepackage{amsthm}

\usepackage[capitalize,noabbrev]{cleveref}

\theoremstyle{plain}
\newtheorem{theorem}{Theorem}[section]
\newtheorem{proposition}[theorem]{Proposition}

\theoremstyle{definition}

\theoremstyle{remark}

\usepackage[textsize=tiny]{todonotes}

\usepackage{multirow}
\usepackage{wrapfig}
\usepackage{lipsum}
\usepackage{setspace}
\usepackage{tabularx}
\usepackage{float}
\usepackage{pifont}

\usepackage{amsmath,amsfonts,bm}

\def\eqref#1{equation~\ref{#1}}

\def\1{\bm{1}}

\def\rvv{{\mathbf{v}}}

\def\rvx{{\mathbf{x}}}

\def\vb{{\bm{b}}}
\def\vc{{\bm{c}}}

\def\vu{{\bm{u}}}
\def\vv{{\bm{v}}}
\def\vw{{\bm{w}}}
\def\vx{{\bm{x}}}

\def\vz{{\bm{z}}}

\DeclareMathAlphabet{\mathsfit}{\encodingdefault}{\sfdefault}{m}{sl}
\SetMathAlphabet{\mathsfit}{bold}{\encodingdefault}{\sfdefault}{bx}{n}

\newcommand{\E}{\mathbb{E}}

\newcommand{\R}{\mathbb{R}}

\newcommand{\xdim}{{D}}
\newcommand{\datasetdim}{{M}}

\newcommand{\estdim}{{K}}

\newcommand{\timedim}{{T}}
\newcommand{\hiddendim}{{H}}

\newcommand{\Rd}{\mathbb{R}^\xdim}

\newcommand{\Rone}{\mathbb{R}}

\newcommand{\pzero}{p_0}

\newcommand{\psigma}{p}
\newcommand{\pv}{p_\rvv}
\newcommand{\tx}{x}
\newcommand{\vtx}{\vx}

\newcommand{\pssigma}{p_\sigma}
\newcommand{\vttx}{\tilde{\vx}}

\newcommand{\loss}{\mathcal{L}}
\newcommand{\esm}{\mathcal{L}_\mathrm{ESM}}
\newcommand{\ism}{\mathcal{L}_\mathrm{ISM}}
\newcommand{\dsm}{\mathcal{L}_\mathrm{DSM}}
\newcommand{\ssm}{\mathcal{L}_\mathrm{SSM}}
\newcommand{\sm}{\mathcal{L}_\mathrm{SM}}
\newcommand{\total}{\mathcal{L}_\mathrm{Total}}
\newcommand{\reg}{\mathcal{L}_\mathrm{QC}}
\newcommand{\regtrace}{\reg^{\mathrm{tr}}}
\newcommand{\regest}{\reg^{\mathrm{est}}}
\newcommand{\regestt}{\tilde{\mathcal{L}}_\mathrm{QC}^{\mathrm{est}}}

\newcommand{\N}{\mathcal{N}}

\newcommand{\half}{\frac{1}{2}}

\newcommand{\grad}[2]{\frac{\partial #1}{\partial #2}}
\newcommand{\gradsec}[3]{\frac{\partial^2 #1}{\partial #2 \partial #3}}
\newcommand{\gradop}[1]{\frac{\partial}{\partial #1}}
\newcommand{\trace}[1]{\mathrm{tr}\left( #1 \right)}
\newcommand{\norm}[1]{\left \lVert #1 \right \rVert}
\newcommand{\forbnorm}[1]{\left \lVert #1 \right \rVert_F}
\newcommand{\exponential}[1]{\mathrm{exp} \left( #1 \right)}

\newcommand{\curlop}[1]{\overline{\mathrm{ROT}_{ij}} #1}
\newcommand{\curlbop}[1]{\overline{\mathrm{ROT}_{ij}} #1}
\newcommand{\stopgradop}[1]{\mathrm{sg}\left[ #1\right]}

\newcommand{\weightr}{R}
\newcommand{\weightw}{W}
\newcommand{\identity}{I}

\newcommand{\jacob}{J}
\newcommand{\param}{\theta}
\newcommand{\asym}{\textit{Asym}}
\newcommand{\normasym}{\textit{NAsym}}
\newcommand{\scoremodel}{s}
\newcommand{\energymodel}{E}

\icmltitlerunning{On Investigating the Conservative Property of Score-Based Generative Models}

\begin{document}

\twocolumn[
\icmltitle{On Investigating the Conservative Property of \\
Score-Based Generative Models}



\icmlsetsymbol{equal}{*}

\begin{icmlauthorlist}
\icmlauthor{Chen-Hao Chao}{nthu}
\icmlauthor{Wei-Fang Sun}{nthu,comp}
\icmlauthor{Bo-Wun Cheng}{nthu}
\icmlauthor{Chun-Yi Lee}{nthu}
\end{icmlauthorlist}

\icmlaffiliation{nthu}{Elsa Lab, Department of Computer Science, National Tsing Hua University, Taiwan.}
\icmlaffiliation{comp}{NVIDIA AI Technology Center, NVIDIA Corporation}

\icmlcorrespondingauthor{Chun-Yi Lee}{cylee@cs.nthu.edu.tw}

\icmlkeywords{Machine Learning, ICML}

\vskip 0.3in
] 


\printAffiliationsAndNotice{}  

\begin{abstract}
Existing Score-Based Models (SBMs) can be categorized into constrained SBMs (CSBMs) or unconstrained SBMs (USBMs) according to their parameterization approaches. CSBMs model probability density functions as Boltzmann distributions, and assign their predictions as the negative gradients of some scalar-valued energy functions. On the other hand, USBMs employ flexible architectures capable of directly estimating scores without the need to explicitly model energy functions. In this paper, we demonstrate that the architectural constraints of CSBMs may limit their modeling ability. In addition, we show that USBMs' inability to preserve the property of \textit{conservativeness} may lead to degraded performance in practice. To address the above issues, we propose Quasi-Conservative Score-Based Models (QCSBMs) for keeping the advantages of both CSBMs and USBMs. Our theoretical derivations demonstrate that the training objective of QCSBMs can be efficiently integrated into the training processes by leveraging the Hutchinson's trace estimator. In addition, our experimental results on the CIFAR-10, CIFAR-100, ImageNet, and SVHN datasets validate the effectiveness of QCSBMs. Finally, we justify the advantage of QCSBMs using an example of a one-layered autoencoder.
\end{abstract}
\section{Introduction}
\label{sec:introduction}

Score-Based Models (SBMs) are parameterized functions for estimating scores, which are vector fields corresponding to the gradients of log probability density functions. Based on their parameterization, SBMs can be categorized into constrained or unconstrained SBMs~\citep{salimans2021should}. 

Constrained SBMs (CSBMs), also known as Energy-Based Models (EBMs), model probability density functions as Boltzmann distributions, and assign their predictions as the negative gradients of some scalar-valued energy functions~\citep{salimans2021should}. CSBMs are able to ensure the \textit{conservativeness} of their output vector fields. This property is essential in guaranteeing that each updates in the sampling process are determined based on the probability ratio between two consecutive sampling steps~\citep{salimans2021should}. This, in turn, is necessary to ensure that the sample distribution converges to the true data distribution. Such a concept has been explored by the researchers of~\citep{pmlr-v32-cheni14, alain2014regularized, Nguyen2017PlugP, salimans2021should}. However, the parameterization of CSBMs requires specific designs, limiting the choices of model architectures for SBMs. For example, the authors of \citep{vincent2011connection, kamyshanska2013autoencoder, saremi2019approximating} proposed to construct an SBM as a neural network with symmetric weights in its linear layer, which hinders its ability to be extended to more sophisticated architectures. On the other hand, the authors of~\citep{salimans2021should, saremi2018deep, song2019sliced} divided an SBM into two halves: the first half explicitly parameterizes the negative energy function, while the second half is generated by automatic differentiation tools~\citep{martens2012estimating} to output the estimated scores. Nevertheless, these methods require that the output of the first half can only be a scalar, and the second half has to be generated using automatic differentiation tools.

In contrast, unconstrained SBMs (USBMs) employ flexible architectures capable of directly estimating the scores without explicitly modeling the energy functions. Due to their architectural flexibility, USBMs have been extensively utilized in contemporary machine learning tasks such as image generation~\citep{song2019generative,ho2020denoising,song2020improved,song2021scorebased,nichol2021improved}. Among these works, \citet{song2021scorebased} proposed a unified framework based on a USBM, which achieved superior performance on several benchmarks. Their success demonstrated that architectural flexibility can be beneficial for SBMs. However, in spite of the empirical benefit of employing USBMs, recent research~\citep{Karras2022edm} suggests that the non-conservativeness of a USBM can cause detrimental effects on its sampling quality. In addition, our analyses in Section~\ref{sec:motivational_example} indicate that USBMs' inability to ensure conservativeness may lead to degraded sampling performance. 

To preserve both the conservativeness of CSBMs and the architectural flexibility of USBMs, we propose Quasi-Conservative Score-Based Models (QCSBMs). Instead of constraining the model architecture, QCSBMs resort to enhancing the conservativeness of USBMs through a regularization loss. Our theoretical derivations demonstrate that such a regularization term can be integrated into the training processes of SBMs efficiently through the Hutchinson's trace estimator~\citep{hutchinson1989stochastic}. Moreover, our experimental results showcase that the performance of Noise Conditional Score Network++ (NCSN++)~\citep{song2021scorebased} can be further improved by incorporating our regularization method on the CIFAR-10, CIFAR-100~\citep{krizhevsky2009learning}, ImageNet-32x32~\citep{van2016pixel}, and SVHN~\citep{Netzer2011} datasets.
\section{Background and Related Works}
\label{sec:background}
In this section, we walk through the background material and the related works for understanding the contents of this paper. We first introduce a number of score matching methods for training an SBM. Next, we describe the sampling algorithms for generating samples through an SBM. Lastly, we elaborate on the conservative property of SBMs, and the differences between CSBMs and USBMs.

\subsection{Score Matching Methods}
\label{sec:background:score_matching}
Score matching~\citep{JMLR:v6:hyvarinen05a} describes the learning process to approximate the score function $\gradop{\vtx} \log\psigma(\vtx)$ using a neural network $\scoremodel(\cdot\ ;\param):\Rd\to \Rd$, which is parameterized by $\param$ and is trained through minimizing the Explicit Score Matching (ESM) objective expressed as follows:
\begin{equation} 
\label{eq:esm}
\esm(\param) = \E_{\psigma(\vtx)} \left[ \half \norm{\scoremodel(\vtx;\param) - \grad{\log \psigma(\vtx)}{\vtx}}^2 \right].
\end{equation}

Eq.~(\ref{eq:esm}) involves explicit evaluation of the true score function $\gradop{\vtx} \log\psigma(\vtx)$, which is intractable for learning tasks without the true probability density function $\psigma(\vtx)$. To address this issue, an alternative method called Implicit Score Matching (ISM)~\citep{JMLR:v6:hyvarinen05a}, which excludes $\gradop{\vtx} \log\psigma(\vtx)$ in the training objective, was introduced to train $\scoremodel(\vtx;\param)$. ISM employs an equivalent loss $\ism$ expressed as follows:
\begin{equation} 
\label{eq:ism}
\ism(\param) = \E_{\psigma(\vtx)} \left[ \half\norm{\scoremodel(\vtx;\param)}^2 + \trace{\grad{\scoremodel(\vtx;\param)}{\vtx}} \right],
\end{equation}

where $\gradop{\vtx}\scoremodel(\vtx;\param)$ corresponds to the Jacobian matrix of $\scoremodel(\vtx;\param)$, and $\trace{\cdot}$ denotes the trace of a matrix. Although $\ism$ avoids the evaluation of $\gradop{\vtx} \log\psigma(\vtx)$, the explicit calculation of $\trace{\gradop{\vtx}\scoremodel(\vtx;\param)}$ in Eq.~(\ref{eq:ism}) still requires $\xdim$ times of backpropagations~\citep{song2019sliced}, which hinders $\ism$'s ability of being utilized in high-dimensional context. To alleviate it, a scalable objective, called Sliced Score Matching (SSM)~\citep{song2019sliced} loss, was proposed to approximate $\trace{\gradop{\vtx}\scoremodel(\vtx;\param)}$ in $\ism$ with the Hutchinson's trace estimator~\citep{hutchinson1989stochastic}.
Given a vector $\vv$ drawn from a distribution $\pv(\vv)$ satisfying $\E_{\pv(\vv)}\left[\vv\vv^T\right]=\identity$, the Hutchinson trace estimator replaces the trace of a square matrix $A$ with $\E_{\pv(\vv)}\left[\vv^T A\vv\right]$, which can be derived as:
\begin{equation} 
\label{eq:trace_est}
\begin{aligned}
\trace{A} &= \trace{A\identity} = \trace{A\E_{\pv(\vv)}[\vv\vv^T]} \\
&= \E_{\pv(\vv)} [\trace{A\vv\vv^T}] = \E_{\pv(\vv)} [\vv^TA\vv].
\end{aligned}
\end{equation}
The above derivation suggests that $\trace{\gradop{\vtx}\scoremodel(\vtx;\param)}$ in Eq.~(\ref{eq:ism}) can be substituted with $\E_{\pv(\vv)}[\vv^T \gradop{\vtx}\scoremodel(\vtx;\param) \vv]$, resulting in an equivalent objective $\ssm$ expressed as follows:
\begin{equation} 
\label{eq:ssm}
\ssm(\param) = \E_{\psigma(\vtx)\pv(\vv)} \left[ \half \norm{\scoremodel(\vtx;\param)}^2 +\vv^T \grad{\scoremodel(\vtx;\param)}{\vtx} \vv \right].
\end{equation}

The vector-Jacobian product $\vv^T \gradop{\vtx}\scoremodel(\vtx;\param)$ can be calculated with a single backward propagation using automatic differentiation~\citep{martens2012estimating}, and the expectation can be approximated using $\estdim$ independently sampled vectors $\{\vv^{(i)}\}_{i=1}^{\estdim}$. Therefore, the computation of $\E_{\pv(\vv)} \left[ \vv^T \gradop{\vtx}\scoremodel(\vtx;\param) \vv\right]$ in Eq.~(\ref{eq:ssm}) can be less expensive than $\trace{\gradop{\vtx}\scoremodel(\vtx;\param)}$ in Eq.~(\ref{eq:ism}) when $\estdim \ll \xdim$. The Denoising Score Matching (DSM) loss is another scalable objective formulated based on the Parzen density estimator~\citep{vincent2011connection}, which further prevents the computational overhead incurred by the gradient operation in $\ssm$:
\begin{equation} 
\label{eq:dsm}
\dsm(\param) 
= \E_{\pssigma(\vttx|\vx)p(\vx)} \left[ \half \norm{\scoremodel(\vttx;\param) - \grad{\log \psigma(\vttx|\vx)}{\vttx}}^2 \right],
\end{equation}
where $\pssigma(\vttx|\vx)\triangleq \frac{1}{(2\pi)^{\xdim/2}\sigma^\xdim}e^{\frac{-1}{2\sigma^2}\norm{\vttx-\vx}^2}$ is is an isotropic Gaussian smoothing kernel with a standard deviation $\sigma$, and $\gradop{\vttx}\log \psigma(\vttx|\vx)=\frac{1}{\sigma^2} (\vx -\vttx)$. Since the computational cost of $\dsm$ is relatively low in comparison to the other score matching losses, it has been widely adopted in contemporary modeling methods~\citep{song2019generative,song2020improved,song2021scorebased, song2021maximum} that pursue training efficiency.

\subsection{Sampling Process}
\label{sec:background:sampling}
Given an optimal SBM $\scoremodel(\vtx;\param)=\gradop{\vtx} \log \psigma(\vtx),\,\forall \vtx\in\Rd$ which minimizes the score-matching objectives (i.e., Eqs.~(\ref{eq:esm}),~(\ref{eq:ism}),~(\ref{eq:ssm}), and~(\ref{eq:dsm})), Langevin dynamics~\citep{bj/1178291835, roberts1998optimal} enables $\psigma(\vtx)$ to be iteratively approximated through the following equation:
\begin{equation} 
\label{eq:langevin}
\vtx_{t+1} = \vtx_{t} + \alpha \scoremodel(\vtx;\param) + \sqrt{2\alpha} \vz_t,
\end{equation}
where $\alpha$ is the step size, $t$ is the timestep, $\vz_t\in \Rd$ is a noise vector sampled from a normal distribution $\N(0,I)$. Under the condition where $\alpha \rightarrow 0$ and $\timedim\rightarrow \infty$, $\vtx_\timedim$ can be generated as if it is directly sampled from $\psigma(\vtx)$~\citep{roberts1998optimal, welling2011bayesian}. Despite the theoretical guarantee of Langevin dynamics, it empirically suffers from the slow mixing issue as discussed by~\citep{song2019generative}, which limits its ability of being utilized in practical data generation scenarios. To resolve this issue, \citet{song2021scorebased} proposed to extend Eq.~(\ref{eq:langevin}) to a time-inhomogeneous variant by making the noise scale $\sigma$, the score model $\scoremodel(\cdot\ ;\param)$, and step size $\alpha$ dependent on $t$. Specifically, they consider a continuous sampling process defined using a stochastic differential equation as follows: 
\begin{equation} 
\label{eq:sde}
    d\vtx = [\mathbf{f}(\vtx,t) - g(t)^2 \scoremodel(\vtx,t;\theta) ]dt + g(t) d\bar{\vw},
\end{equation}
where $dt$ is an infinitesimal negative timestep, $\bar{\vw}$ represents the Wiener process, $\mathbf{f}(\cdot, t)$ is the drift coefficient, and $g(t)$ is the diffusion coefficient. Contemporary score-based generation frameworks~\citep{ho2020denoising, song2021scorebased, song2021maximum, nichol2021improved, Xu2022PoissonFG} implement such a sampling process in two different ways according to the discretization method used. One branch of them~\citep{ho2020denoising, song2021scorebased} follows the concept of Eq.~(\ref{eq:langevin}) to discretize Eq.~(\ref{eq:sde}) using equal-sized steps. The other branch of them~\citep{song2021maximum, Xu2022PoissonFG} leverages an ordinary differential equation (ODE) solver to solve the deterministic variant of Eq.~(\ref{eq:sde}) using adaptive step sizes.

\subsection{Conservativeness and Rotation Density of a Score-Based Model}
\label{sec:background:conservativeness}
A vector field is said to be \textit{conservative} if it can be written as the gradient of a scalar-valued function~\citep{im2016conservativeness}. As proved in~\citep{im2016conservativeness}, the output vector field of an SBM $\scoremodel(\cdot\, ;\param)$ is said to be conservative over a smooth and simply-connected domain $\mathbb{S}\subseteq \Rd$ if and only if its Jacobian is symmetry for all $\vtx \in \mathbb{S}$, which can be equivalently expressed as the zero-rotation-density condition expressed as follows:
\begin{equation} 
\label{eq:curl}
\curlop{\scoremodel(\vtx;\param)} \triangleq \grad{\scoremodel(\vtx;\param)_i}{\vtx_j} - \grad{\scoremodel(\vtx;\param)_j}{\vtx_i} = 0,
\end{equation}
where $1 \leq i,j \leq \xdim$, $\overline{\mathrm{ROT}_{ij}}$ is the rotation density operator~\citep{Glotzl2020HelmholtzDA}, and $\gradop{\vtx_j} \scoremodel(\vtx;\param)_i$ corresponds to the gradient of the $i$-th element of $\scoremodel(\vtx;\param)$ with respect to the $j$-th element of $\vtx$. $\curlop{\scoremodel(\vtx;\param)}$ in Eq.~(\ref{eq:curl}) describes the infinitesimal circulation of $\scoremodel(\vtx ;\param)$ around $\vtx$.

For CSBMs, $\psigma(\vtx)$ is modeled as a Boltzman distribution $p(\vtx;\param)=\exponential{-\energymodel(\vtx;\param)}/Z(\param)$, where $\exponential{\cdot}$ indicates the exponential function, $\energymodel(\cdot\ ;\param):\Rd \to \Rone$ represents a scalar-valued energy function, and $Z(\param)$ refers to the partition function. Therefore, the output vector field of a CSBM can be represented as $\scoremodel(\vtx;\param)=\gradop{\vtx} \log p(\vtx;\param)=-\gradop{\vtx} \energymodel(\vtx;\param)$. This implies that $\scoremodel(\vtx;\param)$ is conservative. In other words, $\scoremodel(\vtx;\param)$ satisfies the zero-rotation-density condition in Eq.~(\ref{eq:curl}), since the mixed second derivatives of $\energymodel(\vtx;\param)$ are equivalent~\citep{alain2014regularized}, which can be shown as the following:
\begin{equation} 
\label{eq:curl_energy}
\curlop{\scoremodel(\vtx;\param)} = \gradsec{\energymodel(\vtx;\param)}{\vtx_j}{\vtx_i} - \gradsec{\energymodel(\vtx;\param)}{\vtx_i}{\vtx_j} = 0.
\end{equation}
Unlike CSBMs, USBMs aim to directly parameterize the true score function $\gradop{\vtx} \log p(\vtx)$ using a vector-valued function $s(\cdot\,;\param):\Rd\to \Rd$. The conservativeness of USBMs is not guaranteed, as $s(\cdot\,;\param)$ does not necessarily correspond to the gradients of a scalar-valued function. Although it is possible to ensure the conservativeness of an USBM under an ideal scenario that $\scoremodel(\vtx;\param)$ perfectly minimizes the score matching target, a trained USBM typically contains approximation errors in practice. This suggests that USBMs are non-conservative in most cases, and do not satisfy the zero-rotation-density condition.
\section{Motivational Examples}
\label{sec:motivational_example}
In this section, we demonstrate the importance of preserving the conservativeness as well as the architectural flexibility of SBMs. In addition, we provide the motivation behind the adoption of QCSBMs through two experiments. 

\begin{figure*}[t]
    \centering
    \includegraphics[width=\linewidth]{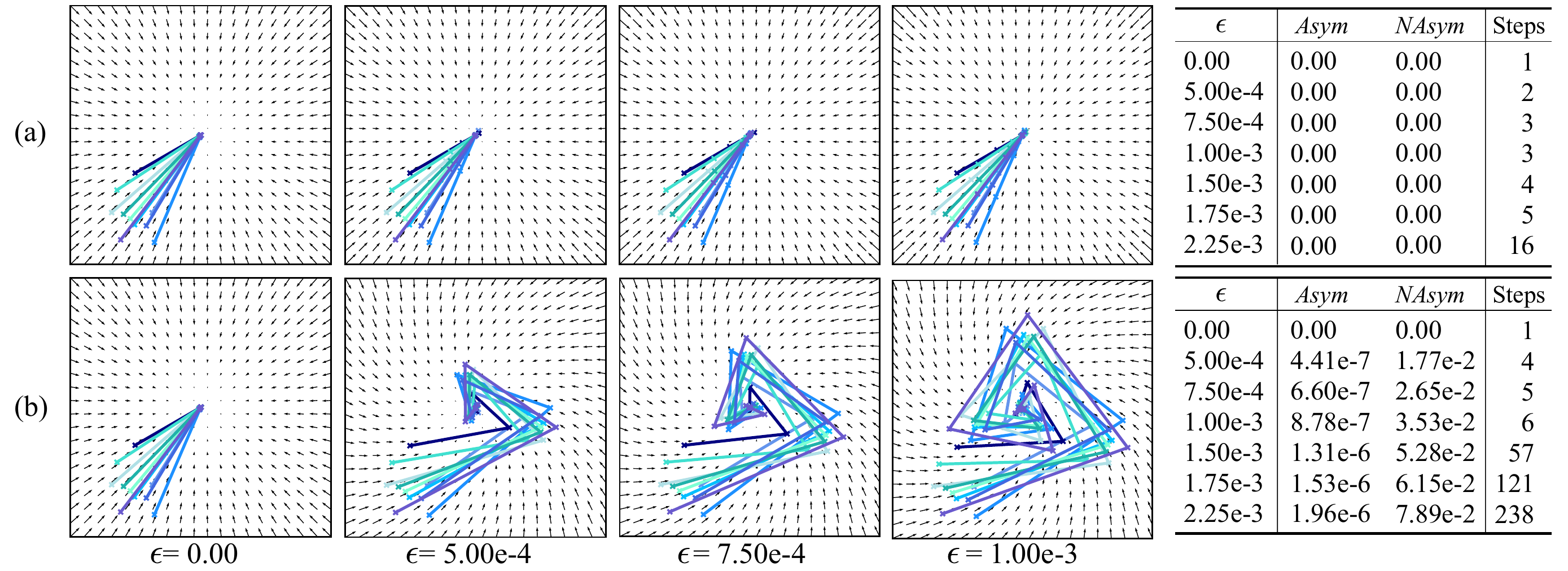}
    \vspace{-1.5em}
    \caption{The visualized examples of (a) the conservative $\scoremodel_C$ and (b) the non-conservative $\scoremodel_U$ under different choices of $\epsilon$. The table on the right-hand side reports the results measured using the $\asym$ and $\normasym$ metrics as well as the number of sampling steps. For a better data visualization, the vector fields are normalized with the maximum norm of $\scoremodel_U$ and $\scoremodel_C$ in each plot.}
    \label{fig:langevin_exp}
\end{figure*}
\subsection{The Influences of Non-Conservativeness on Sampling Process}
\label{sec:motivational_example:curl}
The sampling processes described in Section~\ref{sec:background:sampling} are formulated under a premise that $\scoremodel(\vtx;\param)=\gradop{\vtx} \log \psigma(\vtx)$. In practice, however, a trained USBM contains approximation errors, which could lead to its failure in preserving the conservativeness, as discussed in Section~\ref{sec:background:conservativeness}. In this example, we inspect the impact of the non-conservativeness of USBMs on the sampling process by comparing the sampling efficiency of a USBM and a CSBM under the same approximation error $\epsilon$, i.e., $\esm =\epsilon$. To quantitatively evaluate the non-conservativeness of these SBMs, we measure the magnitude of $\curlop{\scoremodel(\vtx;\param)}$ using the asymmetry metric $\asym\in[0,\infty)$ defined as:
\begin{equation} 
\label{eq:asym}
\begin{aligned}
    &\E_{\psigma(\vtx)}\Bigg[ \half\sum_{i,j=1}^\xdim \left(\curlop{\scoremodel(\vtx;\param)}\right)^2\Bigg]\\
    &=\E_{\psigma(\vtx)}\left[\half \forbnorm{\jacob - \jacob^T}^2\right],
\end{aligned}
\end{equation}
where $\jacob=\gradop{\vtx} \scoremodel(\vtx)$, $\forbnorm{\cdot}$ is the Frobenius norm. Under a regularity condition that $p(\vx)>0,\forall \vx\in \Rd$, $\asym$ equals to 0 if and only if $\scoremodel(\cdot\, ;\param)$ satisfies the zero-rotation-density condition (i.e., Proposition~\ref{prop:asym_cons}). We also measure its normalized variant $\normasym\in[0,1]$ defined as $\E_{\psigma(\vtx)}\left[\forbnorm{\jacob - \jacob^T}^2/(4\forbnorm{\jacob}^2)\right]$ and derived in Section~\ref{apx:asymmetry}. To evaluate the efficiency of the sampling process, we calculate the number of steps required for all sample points to move to the target during the sampling process. In this example, the USBM and the CSBM are denoted as $\scoremodel_{\mathrm{U}}$ and $\scoremodel_{\mathrm{C}}$, and constructed based on Eq.~(\ref{eq:scoremodel_epsilon}) presented in Appendix~\ref{apx:configuration:curl}.

For an illustrative purpose, we present the visualization of the sampling processes as well as the evaluation results under different choices of $\epsilon$ for two specific designs of $\scoremodel_{\mathrm{U}}$ and $\scoremodel_{\mathrm{C}}$ in Figs.~\ref{fig:langevin_exp}~(a)~and~(b), respectively. As demonstrated in the visualized trajectories in Fig.~\ref{fig:langevin_exp}~(b), the existence of the non-conservativeness in $\scoremodel_{\mathrm{U}}$ incurs rotational vector fields tangent to the true score function, leading to inefficient updates during the sampling processes. In addition, the evaluation results in terms of $\asym$, $\normasym$, and the number of sampling steps further reveal that $\scoremodel_{\mathrm{U}}$ requires more function evaluations during the sampling process than $\scoremodel_{\mathrm{C}}$ under the same score-matching error $\epsilon$. The above experimental evidences thus demonstrate that the non-conservativeness of a USBM may decelerate the sampling processes, which, in turn, influences the final sampling performance of it.

\vspace{1em}
\begin{table*}[t]
    \renewcommand{\arraystretch}{1.05}
    \newcommand{\boldtoprule}{\midrule[0.7pt]}
    \newcommand{\thline}{\midrule[0.3pt]}
    \centering
    \caption{The evaluation results of CSBMs, USBMs, and QCSBMs in terms of their means and confidence intervals of three independent runs on the `8-Gaussian,' `Spirals,' and `Checkerboard' datasets, which are detailed in Appendix~\ref{apx:configuration:motivational_example}. The arrow symbols $\uparrow$ / $\downarrow$ indicate that higher / lower values correspond to better performance, respectively.}
    \vspace{0.5em}
    \tiny
    \resizebox{\linewidth}{!}{%
    \begin{tabular}{c|c|cccc|cc}
        \boldtoprule
        Dataset & Model& $\asym$ ($\downarrow$) & $\normasym$ ($\downarrow$) & Score Error ($\downarrow$) & NLL ($\downarrow$) & Precision ($\uparrow$) & Recall ($\uparrow$)\\ 
        \thline
        \multirow{3}{*}{8-Gaussian}   & CSBM   & \textbf{0.00$\pm$0.00 e-3}& \textbf{0.00$\pm$0.00 e-3} & 4.34$\pm$0.06 e-1 & 5.24$\pm$0.03 e+0 & 0.9107$\pm$0.0151 & 0.9057$\pm$0.0094\\
                                      & USBM   & 1.06$\pm$0.16 e-2 & 1.42$\pm$0.43 e-3 & \textbf{4.14$\pm$0.13 e-1} & 5.02$\pm$0.03 e+0 & 0.9453$\pm$0.0061 & 0.8969$\pm$0.0191  \\
                                      & QCSBM  & 9.49$\pm$3.07 e-3 & 1.38$\pm$0.41 e-3  & 4.20$\pm$0.07 e-1 & \textbf{5.01$\pm$0.09 e+0} & \textbf{0.9558$\pm$0.0022} & \textbf{0.9116$\pm$0.0154}\\
        \thline
        \multirow{3}{*}{Spirals}      & CSBM   & \textbf{0.00$\pm$0.00 e-1}& \textbf{0.00$\pm$0.00 e-2} & 1.59$\pm$0.05 e+0 & 5.61$\pm$0.17 e+0 & 0.6221$\pm$0.0300 & 0.7725$\pm$0.0624\\
                                      & USBM   & 7.35$\pm$0.41 e-1 & 8.19$\pm$1.82 e-2  & \textbf{1.50$\pm$0.18 e+0}  & 5.11$\pm$0.11 e+0 & 0.5911$\pm$0.0573 & 0.8230$\pm$0.0119 \\
                                      & QCSBM  &  3.18$\pm$0.16 e-1 & 5.73$\pm$1.64 e-2 & 1.56$\pm$0.08 e+0 & \textbf{5.04$\pm$0.04 e+0} & \textbf{0.6489$\pm$0.0167} & \textbf{0.8244$\pm$0.0204}\\
        \thline
        \multirow{3}{*}{Checkerboard} & CSBM  & \textbf{0.00$\pm$0.00 e-2}& \textbf{0.00$\pm$0.00 e-2} & 7.65$\pm$0.36 e-1 & 5.19$\pm$0.08 e+0 & 0.8990$\pm$0.0072 & 0.9217$\pm$0.0309\\
                                      & USBM  & 1.05$\pm$0.18 e-1 & 2.51$\pm$0.43 e-2 & \textbf{6.79$\pm$0.27 e-1} & 5.09$\pm$0.02 e+0 & 0.9209$\pm$0.0089 & 0.9409$\pm$0.0302\\
                                      & QCSBM & 7.06$\pm$0.43 e-2 & 1.70$\pm$0.18 e-2 & \textbf{6.79$\pm$0.26 e-1} & \textbf{5.08$\pm$0.02 e+0}& \textbf{0.9216$\pm$0.0027} & \textbf{0.9496$\pm$0.0156}\\
        \boldtoprule
    \end{tabular}}
    \label{tab:toy}
\end{table*}
\subsection{The Impacts of Architectural Flexibility on Modeling Ability and Sampling Performance}
\label{sec:motivational_example:flexibility}
To ensure the conservative property of an SBM, previous works~\citep{saremi2018deep, salimans2021should} proposed to construct the architecture such that its output vector field can be described as the gradients of a scalar-valued function. This design, however, potentially limits the modeling ability of an SBM. In this experiment, we empirically examine the influence of architectural flexibility on both the training and sampling processes. For a fair evaluation, a USBM $\scoremodel_\mathrm{U}$ and a CSBM $\scoremodel_\mathrm{C}$ are implemented as neural networks consisting of the same number of parameters. Following the approach described in~\citep{salimans2021should}, these two models are represented as follows:
\begin{equation} 
\label{eq:conversion1}
\begin{aligned}
\scoremodel_\mathrm{U}(\vtx,t;\param_\mathrm{U})&=\frac{1}{\sigma_t} (\vtx-f(\vtx,t;\param_\mathrm{U})),\\
\scoremodel_\mathrm{C}(\vtx,t;\param_\mathrm{C}) &= -\frac{1}{2\sigma_t} \grad{\left \lVert \vtx - f(\vtx,t;\param_\mathrm{C})\right\rVert^2}{\vtx},
\end{aligned}
\end{equation}
where $f:\Rd \to\Rd$ is a neural network, and $\param_\mathrm{U}$ and $\param_\mathrm{C}$ are the parameters. The former is a USBM similar to that used in \citep{song2019generative}, while the latter corresponds to its conservative variant explored by \citet{salimans2021should}. We then compare the conservativeness, the modeling ability, and the sampling performance of both $\scoremodel_\mathrm{U}$ and $\scoremodel_\mathrm{C}$, which are trained independently on three two-dimensional datasets. The conservativeness is measured using $\asym$ and $\normasym$. The modeling ability of an SBM is evaluated based on its score-matching and likelihood-matching abilities, which are quantified using the score-matching error ($\esm$) and the negative log likelihood (NLL) metric, respectively. The sampling performance is evaluated using the Precision and Recall metrics~\citep{Kynknniemi2019ImprovedPA}, which measures the distances between the true samples and the generated samples based on $k$-nearest neighbor algorithm. 

Table~\ref{tab:toy} reports the results of the above setting. The columns `Score Error' and `NLL' in Table~\ref{tab:toy} demonstrate that the USBMs consistently deliver better modeling performance in comparison to the CSBMs, suggesting that their architectural flexibility is indeed beneficial to the training process. On the other hand, due to the potential impact of their non-conservativeness, USBMs are unable to consistently achieve superior results on the precision and recall metrics, as shown in the last two columns of Table~\ref{tab:toy}. The above observations thus indicate that the architectural flexibility of a USBM is crucial to its score-matching and likelihood-matching abilities. Nevertheless, its non-conservativeness may cause negative impacts on its sampling performance. 

The experimental clues in Sections~\ref{sec:motivational_example:curl}~and~\ref{sec:motivational_example:flexibility} shed light on two essential issues to be further explored and addressed. First, although USBMs benefit from their architectural flexibility, their non-conservativeness may lead to degraded sampling performance. Second, despite that CSBMs are conservative, their architectural requirement may limit their modeling abilities in practice. Based on the above observations, this paper intends to investigate a new type of SBMs, called Quasi-Conservative Score-based Models (QCSBMs), which are developed to maintain both the conservativeness as well as the architectural flexibility. As revealed in Table~\ref{tab:toy}, QCSBMs are able to achieve improved results in terms of their conservativeness without sacrificing its modeling ability. In the next section, we elaborate on the formulation and implementation of QCSBMs. 

\vspace{1.2em}
\section{Methodology}
\label{sec:methodology}
In this section, we introduce QCSBMs and present an efficient implementation of them. In Section~\ref{sec:methodology:regularization}, we describe the learning objective of QCSBMs, and derive its scalable variant. In Section~\ref{sec:methodology:implementation}, we detail the training procedure for QCSBMs, and discuss our implementation of its forward and backward propagation processes.


\begin{figure*}[t]
    \vspace{-0.5em}
    \centering
    \includegraphics[width=\linewidth]{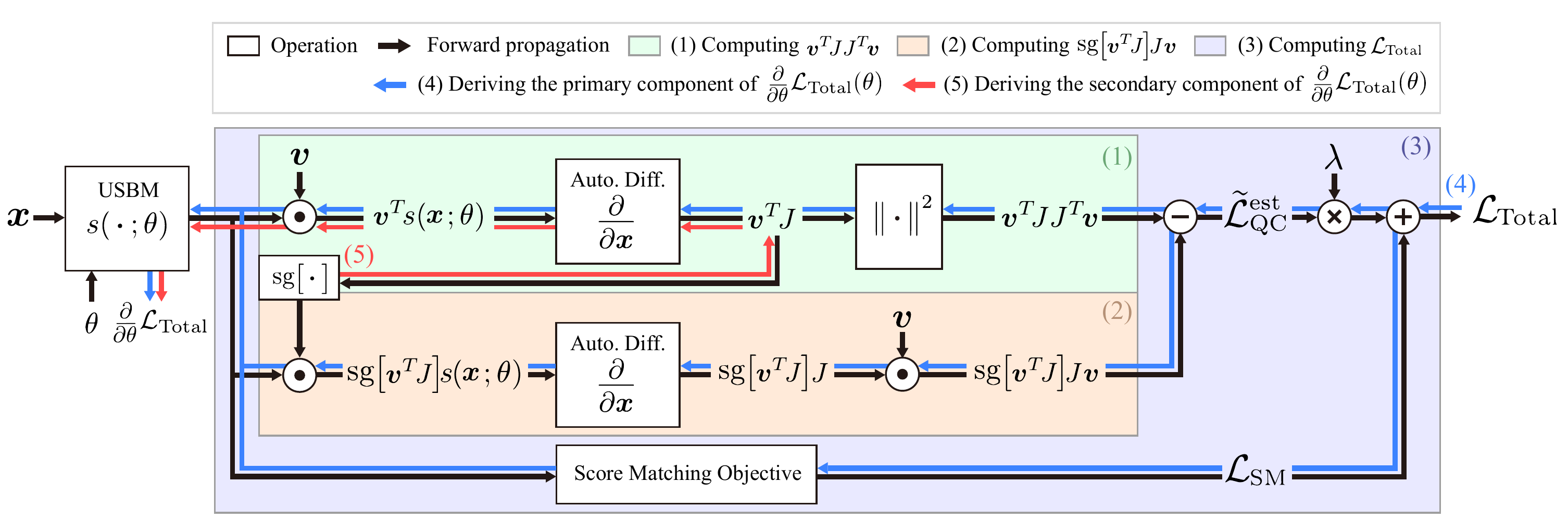}
    \vspace{-2.5em}
    \caption{The computational graph of $\total$ in QCSBMs. The `Auto. Diff.' blocks represent the operation of differentiating $\vu^T \scoremodel(\vtx;\param)$, where $\vu$ is a constant vector with respect to $\vtx$.}
    \label{fig:training}
    \vspace{-0.8em}
\end{figure*}
\subsection{Quasi-Conservative Score-Based Models}
\label{sec:methodology:regularization}
Instead of following the concept of CSBMs to ensure the conservativeness through architectural constraints, QCSBMs resort to penalizing the non-conservativeness through a regularization loss. The training objective for QCSBMs is defined as $\total$, which is expressed as follows:
\begin{equation} 
\label{eq:loss_total}
\total(\param) = \sm(\param) + \lambda \reg(\param),
\end{equation}
where $\sm$ can be any one of the score-matching objectives (i.e., Eqs.~(\ref{eq:esm}),~(\ref{eq:ism}),~(\ref{eq:ssm}), or~(\ref{eq:dsm})), $\reg$ represents the regularization term reflecting the non-conservativeness, and $\lambda$ is a balancing factor. As discussed in Section~\ref{sec:motivational_example:curl}, the non-conservativeness of a USBM can be measured using the magnitude of its rotation densities in the Frobenius norm (i.e., Eq.~(\ref{eq:asym})), suggesting a formulation of $\reg$ as:
\begin{equation} 
\label{eq:reg_qc}
\reg(\param) = \E_{\psigma(\vtx)} \left[ \half \norm{\jacob - \jacob^T}^2_F \right],
\end{equation}
where $\jacob = \gradop{\vtx} \scoremodel(\vtx;\param)$. This objective function, however, requires $\xdim$ times of backpropagations to explicitly calculate the Jacobian matrix of $\scoremodel(\vtx;\param)$. In order to reduce the computational cost, we first derive an equivalent objective $\regtrace$, and then utilize the Hutchinson's trace estimator to approximate it. The loss $\regtrace$ is derived in Appendix~\ref{apx:derivations:reg_trace}, and formulated as follows:
\begin{equation} 
\label{eq:reg_qc_trace}
\regtrace(\param) = \E_{\psigma(\vtx)} \left[ \trace{\jacob \jacob^T} - \trace{ \jacob \jacob} \right].
\end{equation}
By applying the Hutchinson's trace estimator to both $\trace{\jacob \jacob^T}$ and $\trace{\jacob \jacob}$ according to Eq.~(\ref{eq:trace_est}), $\regtrace$ can be re-expressed using an unbiased objective $\regest$, which is defined as the following:
\begin{equation} 
\label{eq:reg_qc_est}
\begin{aligned}
\regest(\param) &= \E_{\psigma(\vtx)} \left[ \E_{\pv(\vv)} \left[\vv^T \jacob \jacob^T \vv \right]- \E_{\pv(\vv)} \left[\vv^T \jacob \jacob \vv \right]\right].
\end{aligned}
\end{equation}

Since $\E_{\pv(\vv)} \left[\vv^T \jacob \jacob^T \vv \right]$ and $\E_{\pv(\vv)} \left[\vv^T \jacob \jacob \vv \right]$ can be simultaneously approximated, $\regest$ in Eq.~(\ref{eq:reg_qc_est}) can be rewritten as a variant $\regestt$ defined as follows:
\begin{equation} 
\label{eq:reg_qc_estt}
\begin{aligned}
\regestt(\param) &= \E_{\psigma(\vtx)} \left[ \E_{\pv(\vv)} \left[\vv^T \jacob \jacob^T \vv -\vv^T \jacob \jacob \vv \right]\right].
\end{aligned}
\end{equation}
The implementation of Eq.~(\ref{eq:reg_qc_estt}) can be more efficient than that of Eq.~(\ref{eq:reg_qc_est}) as the vector-Jacobian product $\vv^T \jacob$ in Eq.~(\ref{eq:reg_qc_estt}) can be calculated once and reused in the computation of both $\vv^T \jacob \jacob^T \vv$ and $\vv^T \jacob \jacob \vv$. We provide a detailed description of such an approach in the following section.

\subsection{The Training Procedure of QCSBMs}
\label{sec:methodology:implementation}
\begin{algorithm}[tb]
   \caption{Training Procedure of QCSBM}
   \label{alg:training}
\begin{algorithmic}
\setstretch{1.15}
      \STATE {\bfseries Input:} $\vtx$, $\vv$, $\scoremodel(\cdot\ ;\param)$, $\lambda$ \\
      \STATE \textcolor{gray}{(1) Computing $\vv^T \jacob \jacob^T \vv$.}  \\
      \STATE $\vv^T\jacob\leftarrow \gradop{\vtx} [\vv^T \scoremodel(\vtx;\param)]$ \\
      \STATE $\vv^T \jacob \jacob^T \vv \leftarrow \norm{\vv^T \jacob}^2$ \\
      \vspace{0.5em}
      \STATE \textcolor{gray}{(2) Computing $\vv^T \jacob \jacob \vv$.} \\
      \STATE $\stopgradop{\vv^T \jacob} \jacob \leftarrow \gradop{\vtx} [\stopgradop{\vv^T \jacob} \scoremodel(\vtx;\param)]$ \\
      \STATE $\stopgradop{\vv^T \jacob} \jacob \vv  \leftarrow \stopgradop{\vv^T \jacob} \jacob \cdot \vv$ \\
      \vspace{0.5em}
      \STATE \textcolor{gray}{(3) Computing $\total(\param)$.} \\
      \STATE $\regestt(\param)\leftarrow\vv^T \jacob \jacob^T \vv -\stopgradop{\vv^T \jacob} \jacob \vv$ \\
      \STATE $\sm (\param) \leftarrow$ Eq.~(\ref{eq:esm}),~(\ref{eq:ism}),~(\ref{eq:ssm}), or (\ref{eq:dsm}) \\
      \STATE $\total (\param) \leftarrow \sm(\param)+\lambda\regestt(\param)$ \\
      \vspace{0.5em}
      \STATE \textcolor{gray}{(4) Deriving the primary component of $\frac{\partial}{\partial \param}\total(\param)$.} \\
      \STATE Perform backpropagation through the blue arrows. \\
      \vspace{0.5em}
      \STATE \textcolor{gray}{(5) Deriving the secondary component of $\frac{\partial}{\partial \param}\total(\param)$.} \\
      \STATE Perform backpropagation through the red arrows. \\
      \vspace{0.5em}
      \STATE Update $\param$ with $\frac{\partial}{\partial \param}\total(\param)$.
\end{algorithmic}
\end{algorithm}

In this subsection, we walk through the proposed training procedure of QCSBMs. The training procedure is detailed in Algorithm~\ref{alg:training}, and the corresponding computational graph is illustrated in Fig.~\ref{fig:training}. For the sake of notational simplicity, we assume that both the batch size and the number of random vectors $\estdim$ are 1 in Algorithm~\ref{alg:training} and Fig.~\ref{fig:training}. The entire training procedure is divided into five steps, denoted as Steps (1)$\sim$(5), respectively. Steps (1)$\sim$(3) describe the forward propagation process of $\total(\param)$, which is depicted by the black arrows in Fig.~\ref{fig:training}. 
Steps (4) and (5) correspond to the backpropagation processes of the two gradient components comprising $\gradop{\param}\total(\param)$, which are named the primary and secondary components, and are depicted as the blue and red arrows in Fig.~\ref{fig:training}, respectively. The detailed formulations for these two components and the rationale behind such a two-step backpropagation process are further elaborated in Appendix~\ref{apx:optimization}. Please note that the symbol $\stopgradop{\cdot}$ used in Algorithm~\ref{alg:training} represents the `stop gradient' operation, which is adopted to disconnect the computational graph.

Based on the implementation described in Algorithm~\ref{alg:training}, the computation of $\regestt$ can be more efficient than $\regest$ since it does not require repeatedly calculating the vector-Jacobian product $\vv^T \jacob$ during forward propagation. Moreover, unlike $\reg$ and $\regtrace$, $\regestt$ does not require $\xdim$ times of backward propagation, which justifies its computational efficiency.
\section{Experiments on Real-World Datasets}
\label{sec:experimental_results}
\begin{table}
\renewcommand{\arraystretch}{1.1}
    \newcommand{\boldtoprule}{\toprule[1.1pt]}
    \newcommand{\thline}{\midrule[0.3pt]}
    \centering
    \vspace{-0.8em}
    \caption{The NLL, $\asym$, and $\normasym$ of C-NCSN++, U-NCSN++, and QC-NCSN++ evaluated on the CIFAR-10, CIFAR-100, ImageNet-32x32, and SVHN datasets.}
    \vspace{0.2em}
    \footnotesize
    \resizebox{\linewidth}{!}{
    \begin{tabular}{c|ccc|ccc}
        \boldtoprule
          & \multicolumn{3}{c|}{CIFAR-10} & \multicolumn{3}{c}{ImageNet-32x32} \\
        \thline
        Method   & NLL & $\asym$ & $\normasym$ & NLL & $\asym$& $\normasym$ \\
        \thline
        C-NCSN++  & 5.91 & \textbf{0.00} &  \textbf{0.00}  & 5.10 & \textbf{0.00} & \textbf{0.00}  \\
        U-NCSN++   & 3.46 & 1.88 e8 &  1.90 e-3  & 3.96 &  2.05 e7 &  7.17 e-4  \\
        QC-NCSN++  & \textbf{3.38} & 3.49 e7 & 8.41 e-4 & \textbf{3.83} & 1.13 e7 & 5.47 e-4 \\
        \boldtoprule
        & \multicolumn{3}{c|}{CIFAR-100} & \multicolumn{3}{c}{SVHN} \\
        \thline
        Method   & NLL & $\asym$ & $\normasym$ & NLL & $\asym$& $\normasym$\\
        \thline
        C-NCSN++  & 5.34 & \textbf{0.00} &  \textbf{0.00}  & 5.00 & \textbf{0.00} &  \textbf{0.00}  \\
        U-NCSN++  & 3.50 &  2.98 e8 &  2.25 e-3   & 2.15 &  3.06 e7 &  6.54 e-4  \\
        QC-NCSN++ & \textbf{3.44} & 9.31 e7 & 1.44 e-3 & \textbf{2.01} & 1.69 e7 & 4.80 e-4 \\
        \boldtoprule
    \end{tabular}
    }
    \vspace{-1em}
    \label{tab:nll}
\end{table}
In this section, we examine the effectiveness of the proposed QCSBMs on four real-world datasets: CIFAR-10, CIFAR-100~\citep{krizhevsky2009learning}, ImageNet-32x32~\citep{van2016pixel}, and SVHN~\citep{Netzer2011} datasets. We employ the unconstrained architecture as well as the training procedure adopted by NCSN++~(VE)~\citep{song2021scorebased} as our baseline, and denote this method as `U-NCSN++' in our experiments. On the other hand, C-NCSN++ and QC-NCSN++, which are variants of U-NCSN++ constructed based on Eq.~(\ref{eq:conversion1}) and regularized by $\regestt$, are compared against U-NCSN++ using the NLL, $\asym$, $\normasym$, Fr\'echet Inception Distance (FID)~\citep{heusel2017gans}, and Inception Score (IS)~\citep{barratt2018note}, Precision, and Recall metrics. The details of the experimental setups are provided in Appendix~\ref{apx:configuration:datasets}.

\paragraph{Likelihood and Conservativeness Evaluation.} Table~\ref{tab:nll} reports the evaluation results of U-NCSN++, C-NCSN++, and QC-NCSN++ in terms of NLL, $\asym$, and $\normasym$ on the four real-world datasets. The evaluation results of C-NCSN++ is inferior to those of U-NCSN++ and QC-NCSN++ on the NLL metric, which aligns with our observation in Section~\ref{sec:motivational_example}, suggesting that the modeling flexibility is influential to the final performance on the NLL metric. In addition, we observe that the evaluation results on the NLL metric can be further improved when $\regestt$ is incorporated into the training process. As demonstrated in the table, QC-NCSN++, which achieves superior performance in terms of $\asym$ and $\normasym$ metrics, also has a noticeable improvement on the NLL metric.

\begin{table}[t]
    \renewcommand{\arraystretch}{0.885}
    \newcommand{\boldtoprule}{\midrule[0.9pt]}
    \newcommand{\thline}{\midrule[0.5pt]}
    \centering
    \vspace{0.5em}
    \caption{The sampling performance and NFE of C-NCSN++, U-NCSN++, and QC-NCSN++ with an ODE sampler. The arrow symbols $\uparrow$ / $\downarrow$ indicate that higher / lower values correspond to better performance, respectively. }
    \resizebox{\linewidth}{!}{%
    \begin{tabular}{c|c|cccc}
        \boldtoprule
        Method  & NFE ($\downarrow$) & FID ($\downarrow$) & IS ($\uparrow$) & Prec. ($\uparrow$) & Rec. ($\uparrow$)\\ 
        \boldtoprule
          \multicolumn{6}{c}{CIFAR-10} \\
        \thline
        C-NCSN++  & 343 & 16.66 & 7.96 & 0.57 & 0.60 \\
        U-NCSN++  & 170 & 7.48 & 9.24 & \textbf{0.61} & \textbf{0.62} \\
        QC-NCSN++ & \textbf{124} & \textbf{7.21} & \textbf{9.25} & \textbf{0.61} & \textbf{0.62} \\
        \boldtoprule
        \multicolumn{6}{c}{ImageNet-32x32} \\
        \thline 
        C-NCSN++  & 313 & 24.61 & 8.56 & 0.55 & 0.49 \\
        U-NCSN++  & 148 & 17.09 & 9.80 & 0.55 & \textbf{0.55} \\
        QC-NCSN++ & \textbf{115}& \textbf{16.62} & \textbf{9.85} & \textbf{0.56} & \textbf{0.55} \\
        \boldtoprule
        \multicolumn{6}{c}{CIFAR-100} \\
        \thline
        C-NCSN++ & 297 & 21.89 & 7.84 & 0.55 & 0.56  \\
        U-NCSN++ & 168 & 8.95 & 10.09 & \textbf{0.59} & 0.63  \\
        QC-NCSN++  & \textbf{131} & \textbf{8.90} & \textbf{10.12} & \textbf{0.59} & \textbf{0.64} \\
        \boldtoprule
        \multicolumn{6}{c}{SVHN} \\
        \thline
        C-NCSN++ & 498 & 24.78 & 2.76 & 0.55 & 0.48 \\
        U-NCSN++ &  209 & 16.08 & 3.17 & 0.56 & 0.63 \\
        QC-NCSN++ & \textbf{126}  & \textbf{15.15} & \textbf{3.24} & \textbf{0.59} & \textbf{0.65}  \\
        \boldtoprule
    \end{tabular}}
    \label{tab:fidis_pr_nfe}
    \vspace{-1em}
\end{table}
\begin{table}[t]
    \renewcommand{\arraystretch}{0.92}
    \newcommand{\boldtoprule}{\midrule[0.8pt]}
    \newcommand{\thline}{\midrule[0.3pt]}
    \centering
    \vspace{0.5em}
    \caption{The sampling performance and NFE of C-NCSN++, U-NCSN++, and QC-NCSN++ with the PC sampler. The arrow symbols $\uparrow$ / $\downarrow$ indicate that higher / lower values correspond to better performance, respectively.}
    \footnotesize
    \resizebox{\linewidth}{!}{
    \begin{tabular}{c|c|cccc}
        \boldtoprule
        Method   & NFE & FID ($\downarrow$) & IS ($\uparrow$) & Prec. ($\uparrow$) & Rec. ($\uparrow$) \\ 
        \boldtoprule
        \multicolumn{6}{c}{CIFAR-10} \\
        \thline
        C-NCSN++    & \multirow{3}{*}{1,000} & 10.97 & 8.58 & 0.61 & 0.58\\
        U-NCSN++    & & 2.50 & 9.58 & \textbf{0.67} & \textbf{0.60}  \\
        QC-NCSN++   & & \textbf{2.48} & \textbf{9.70} & \textbf{0.67} & \textbf{0.60} \\
         \boldtoprule
        \multicolumn{6}{c}{ImageNet-32x32} \\
        \thline
        C-NCSN++    & \multirow{3}{*}{1,000} & 28.97 & 8.58 & \textbf{0.61} & 0.45 \\
        U-NCSN++    & & 19.82 & 9.89 & 0.60 & \textbf{0.52}  \\
        QC-NCSN++   & & \textbf{19.62} & \textbf{9.94} & \textbf{0.61} &  \textbf{0.52} \\
        \boldtoprule
        \multicolumn{6}{c}{CIFAR-100} \\
        \thline
        C-NCSN++    & \multirow{3}{*}{1,000} & 17.59 & 8.38 & 0.60 & 0.54 \\
        U-NCSN++    & & 2.54 & 9.63 & 0.60 & 0.66  \\
        QC-NCSN++   & & \textbf{2.45} & \textbf{9.75} & \textbf{0.61} &  \textbf{0.67} \\
        \boldtoprule
        \multicolumn{6}{c}{SVHN} \\
        \thline
        C-NCSN++    & \multirow{3}{*}{1,000} & 24.71 & 2.66 & \textbf{0.61} & 0.46 \\
        U-NCSN++    & & 14.34 & 3.10 & 0.60 & \textbf{0.67}  \\
        QC-NCSN++   & &\textbf{13.88} & \textbf{3.12} & \textbf{0.61} &  \textbf{0.67} \\
        \boldtoprule
    \end{tabular}}
    \label{tab:pc}
    \vspace{-1em}
\end{table}

\paragraph{Sampling with an ODE Solver.} In this experiment, we examine the sampling performance of U-NCSN++, C-NCSN++, and QC-NCSN++ based on the number of function evaluations (NFE) and the FID/IS/Precision/Recall metrics. The sampler is implemented using the RK45~\citep{dormand1980family} ODE solver. Table~\ref{tab:fidis_pr_nfe} presents the evaluation results of the above setting. It is observed that C-NCSN++ is inferior to U-NCSN++ and QC-NCSN++, suggesting that modeling errors can be influential to the sampling performance. On the other hand, QC-NCSN++ performs comparably to U-NCSN++ in terms of the sampling performance metrics with fewer function evaluations, indicating that QC-NCSN++ is able to deliver a better sampling efficiency. 

\paragraph{Sampling under a Fixed NFE.} In this experiment, we compare the sampling performance of C-NCSN++, U-NCSN++, and QC-NCSN++ under a fixed NFE using the Predictor-Corrector (PC) sampler~\citep{song2021scorebased}. Different from the ODE sampler presented above, PC sampler discretizes the sampling process described Eq.~(\ref{eq:sde}) with equal-sized steps according to a predetermined value of NFE. Table~\ref{tab:pc} presents the evaluation results of C-NCSN++, U-NCSN++, and QC-NCSN++ when NFE is equal to 1,000. It is observed that QC-NCSN++ demonstrates improved performance in comparison to C-NCSN++ and U-NCSN++ in terms of the FID/IS/Precion/Recall metrics. The above experimental results validate the benefit of minimizing the non-conservativeness of a USBM.

\begin{figure*}[t]
    \centering
    \vspace{1em}
    \footnotesize
    \includegraphics[width=\linewidth]{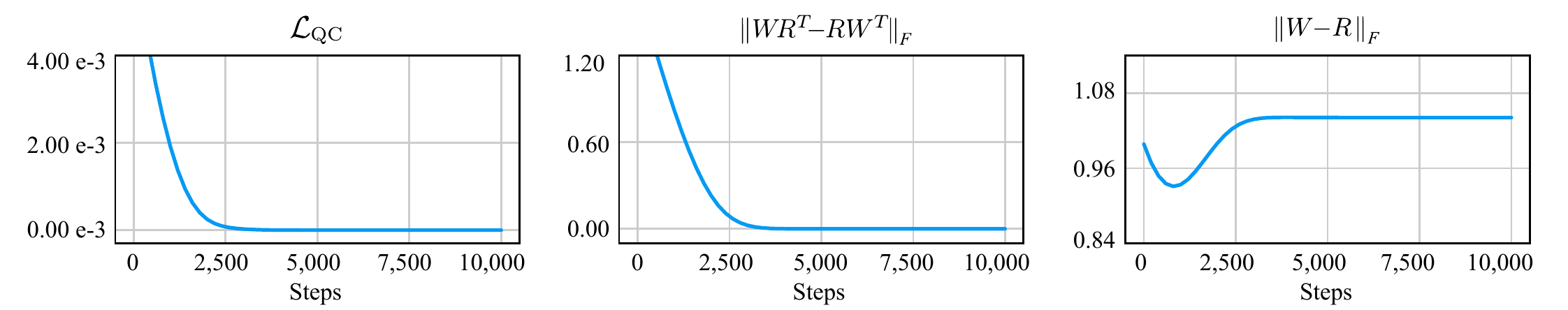}
    \vspace{-2em}
    \caption{The trends of $\forbnorm{\weightw\weightr^T-\weightr\weightw^T}$ and $\forbnorm{\weightw-\weightr}$ during the minimization process of $\reg$. The `steps' on the x-axes refer to the training steps.}
    \label{fig:sym_exp}
\end{figure*}
\begin{figure}
  \begin{center}
    \includegraphics[width=0.5\textwidth]{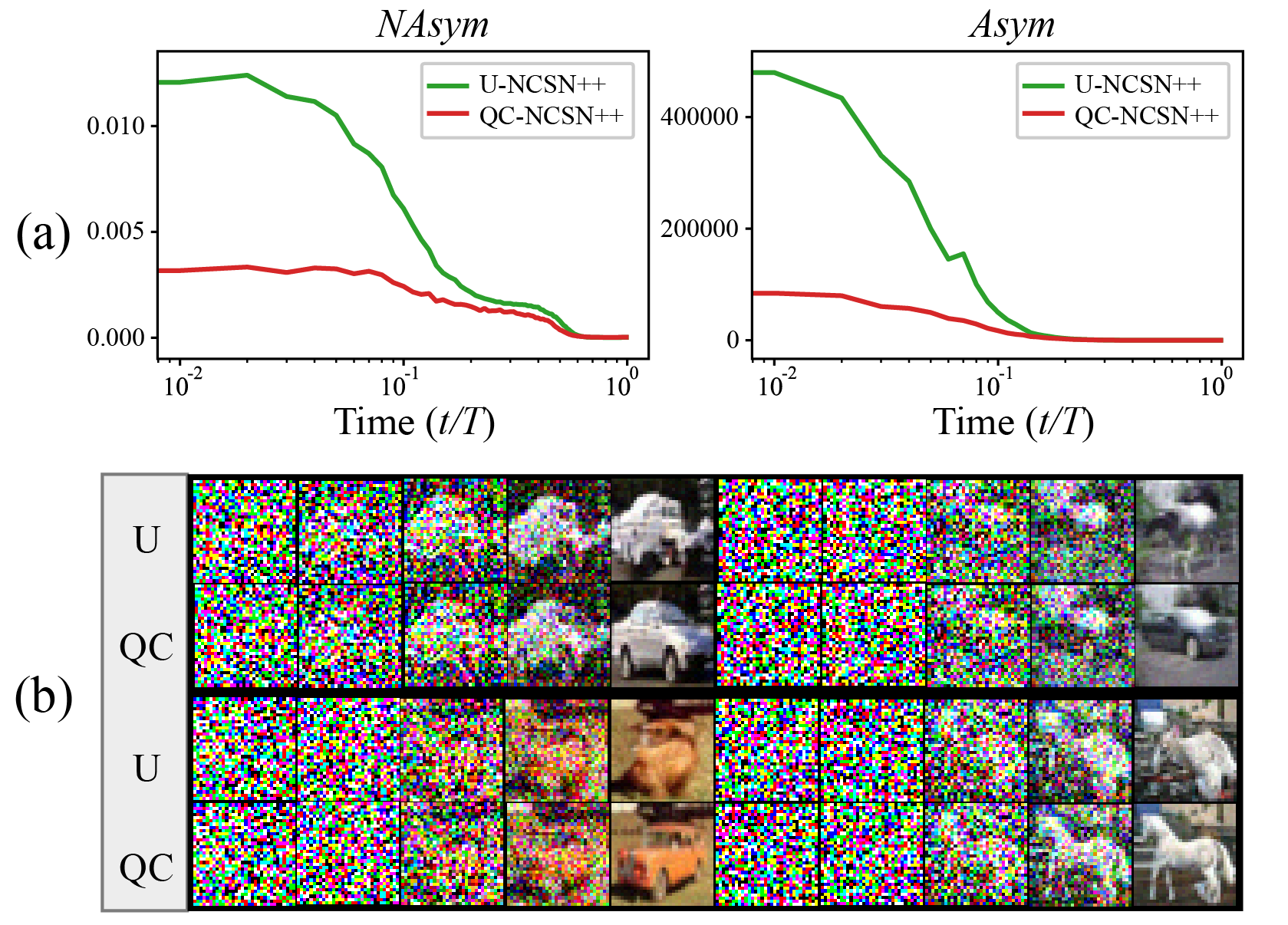}
  \end{center}
  \vspace{-1em}
  \caption{(a) The evaluation results of $\asym$ and $\normasym$ under different timestep $t$ on the CIFAR-10 dataset. (b) Examples generated by U-NCSN++ and QC-NCSN++ with the same random seed.}
  \label{fig:asym_demo}
  \vspace{-1.5em}
\end{figure}

\paragraph{The Effects of Non-Conservativeness during the Sampling Process.} To further investigate the influence of non-conservativeness during the sampling process, we measure $\asym$ and $\normasym$ on the $t$-axis, i.e., the non-conservativeness under different timesteps. As shown in Fig.~\ref{fig:asym_demo}~(a), QC-NCSN++ delivers lower $\asym$ and $\normasym$ under different $t$ in comparison to U-NCSN++. This result suggests that QC-NCSN++ can be less susceptible to its non-conservativeness during the sampling process. In our qualitative comparison presented in Fig.~\ref{fig:asym_demo}~(b), we observe that some examples generated using QC-NCSN++ have noticeably improved sample quality as compared to U-NCSN++.

\section{QCSBM Implemented as a One-Layered Autoencoder}
\label{sec:mechanism}
A line of research \citep{vincent2011connection, kamyshanska2013autoencoder, im2016conservativeness, Kamyshanska2015ThePE} focuses on a type of SBM constructed as a one-layered autoencoder, in which the property of conservativeness can be systematically analyzed. Such an SBM is represented as $\scoremodel(\vtx;\param)= \weightr h(\weightw^T\vtx+\vb)+\vc$, where $h(\cdot)$ is an activation function, $\vb, \vc\in \Rd$, $\weightr, \weightw\in \R^{\xdim\times H}$ are the weights of $\scoremodel(\cdot\ ;\param)$~(i.e., $\theta=\{\weightr,\weightw,\vb,\vc\}$), and $H$ is the width of the hidden-layer.
As proved in~\citep{im2016conservativeness}, the output vector field of $\scoremodel(\cdot\ ;\param)$ is conservative if and only if $\weightw\weightr^T=\weightr\weightw^T$. To ensure the conservativeness of such an SBM, a number of works~\citep{vincent2011connection, kamyshanska2013autoencoder, Kamyshanska2015ThePE} follow the concept of CSBMs and restrict the weights of $\scoremodel(\cdot\ ;\param)$ to be `tied,' i.e., $\weightw=\weightr$.
An SBM with tied weights, however, is only a sufficient condition for its conservativeness, rather than a necessary one. This implies that there must exist some conservative $\scoremodel(\cdot\ ;\param)$ that cannot be modeled using tied weights.

Instead of enforcing an SBM's weights to be tied (i.e., $\weightw=\weightr$), QCSBMs indirectly learn to satisfy the conservativeness  condition (i.e., $\weightw\weightr^T=\weightr\weightw^T$) through minimizing $\reg$. Fig.~\ref{fig:sym_exp} depicts the trends of $\forbnorm{\weightw\weightr^T-\weightr\weightw^T}$ and $\forbnorm{\weightw-\weightr}$ during the minimization process of $\reg$. 
As the training progresses, the values of $\forbnorm{\weightw\weightr^T-\weightr\weightw^T}$ approach zero, indicating that $\scoremodel(\cdot\ ;\param)$ learns to output a conservative vector field through minimizing $\reg$. In contrast, the values of $\forbnorm{\weightw-\weightr}$ do not decrease to zero, revealing that minimizing $\reg$ does not necessarily lead to $\weightw=\weightr$. The experimental results thus suggest that QCSBMs can learn to output conservative vector fields that cannot be modeled by one-layered autoencoders with tied weights. This justifies the advantage of QCSBMs over CSBMs, as QCSBMs provide a more flexible parameterization while still maintaining their conservativeness. In Appendix~\ref{apx:experiments:mechanism}, we offer more examples to support this observation.

\vspace{0.15em}
\section{Conclusion}
\label{sec:conclusion}
In this paper, we unveiled the underlying issues of CSBMs and USBMs, and highlighted the importance of preserving both of the architectural flexibility and the property of conservativeness through two motivational experiments. We proposed a new category of SBMs, named QCSBMs, in which the magnitudes of their rotation densities are minimized through a regularization loss for enhancing their property of conservativeness. We showed that such a regularization loss can be reformulated as a scalable variant based on the Hutchinson's trace estimator, and demonstrated that it can be efficiently incorporated into the training procedure of SBMs. Finally, we validated the effectiveness of QCSBMs through the experimental results on the real-world datasets, and showcased the advantage of QCSBMs over CSBMs using the example of a one-layered autoencoder.

\section*{Acknowledgements}
\label{sec:ack}
The authors gratefully acknowledge the support from the National Science and Technology Council (NSTC) in Taiwan under grant number MOST 111-2223-E-007-004-MY3, as well as the financial support from MediaTek Inc., Taiwan. The authors would also like to express their appreciation for the donation of the GPUs from NVIDIA Corporation and NVIDIA AI Technology Center (NVAITC) used in this work. Furthermore, the authors extend their gratitude to the National Center for High-Performance Computing (NCHC) for providing the necessary computational and storage resources.
\bibliography{citation}
\bibliographystyle{icml2023}

\newpage
\appendix
\onecolumn
\setcounter{section}{0}
\setcounter{equation}{0}
\setcounter{figure}{0}
\setcounter{table}{0}
\setcounter{algorithm}{0}

\renewcommand{\thealgorithm}{A\arabic{algorithm}}
\renewcommand{\thefigure}{A\arabic{figure}}
\renewcommand{\thetable}{A\arabic{table}}
\renewcommand{\theequation}{A\arabic{equation}}

\section{Appendix}
\label{apx}
In this Appendix, we first provide the definitions for the symbols used in the main manuscript and the Appendix in Section~\ref{apx:notation}. Next, we detail the backpropagation processes described in Section~\ref{sec:methodology:implementation}, and provide the derivation for (\ref{eq:reg_qc_trace}) in Section~\ref{apx:derivations}. Then, we offer a discussion on the optimization process of QCSBMs as well as the normalized asymmetry metric in Sections~\ref{apx:optimization}~and~\ref{apx:asymmetry}, respectively.  Subsequently, in Section~\ref{apx:sampling}, we describe the approach to extend a QCSBM to its time-inhomogeneous variant, i.e., QC-NCSN++ described in Section~\ref{sec:experimental_results} of the main manuscript. Finally, we provide the detailed experimental configurations in Section~\ref{apx:configuration}, and a number of qualitative and quantitative experimental results in Section~\ref{apx:experiments}.

\subsection{List of Notations}
\label{apx:notation}
In this section, we offer the list of notations used throughout the main manuscript and the Appendix. These notations and their descriptions are summarized in Table~\ref{tab:notation}.

\begin{table*}[h!]
\renewcommand{\arraystretch}{1.175}
\newcommand{\boldtoprule}{\toprule[1.2pt]}
\centering
\footnotesize
\begin{tabularx}{\textwidth}{l|X}
    \boldtoprule
    Symbol & Description  \\
    \boldtoprule
    $\datasetdim$         & the dataset size. \\
    $\xdim$               & the data dimension. \\
    $\estdim$             & the number of random vectors used in the Hutchinson's trace estimator. \\
    $\timedim$            & the number of discretized timesteps for the sampling algorithm. \\
    $\hiddendim$          & the hidden dimension of the one-layered autoencoder described in Section~\ref{sec:mechanism}. \\
    \hline
    $\alpha$              & the step size used in Langevin dynamics. \\
    $\epsilon$            & the score-matching error described in Section~\ref{sec:motivational_example:curl}. \\
    $\sigma$              & the standard deviation of a Gaussian distribution. \\
    $\param$              & the parameters of an SBM. \\
    \hline
    $\vx\in\Rd$           & a data sample. \\
    $\vttx\in\Rd$          & a perturbed data sample. \\
    $\vz\in\Rd$           & a noise vector used in Langevin dynamics. \\
    $\vv\in\Rd$           & a random vector used in the Hutchinson trace estimator. \\
    $\vb,\vc\in\Rd$                             & the bias of the one-layered autoencoder described in Section~\ref{sec:mechanism}. \\
    $\weightw,\weightr\in\R^{\xdim\times \hiddendim}$                            & the weights of the one-layered autoencoder described in Section~\ref{sec:mechanism}. \\
    \hline
    $Z(\param)$   & a partition function of a Boltzmann distribution. \\
    $\energymodel(\cdot\ ;\param):\Rd \to \R$  & an energy model parameterized by $\param$. \\
    $\scoremodel(\cdot\ ;\param):\Rd \to \Rd$   & a score model parameterized by $\param$. \\
    \hline
    $\gradop{\vtx}\energymodel(\vtx;\param)$  & the gradient of $\energymodel(\vtx;\param)$ w.r.t. $\vtx$. \\
    $\gradop{\vtx}\scoremodel(\vtx;\param)$     & the Jacobian matrix of $\scoremodel(\vtx;\param)$. \\
    $\jacob$                             & the simplified notation for $\gradop{\vtx}\scoremodel(\vtx;\param)$. \\
    \hline
    $\esm$ & Explicit Score Matching (ESM) loss defined in Eq.~(\ref{eq:esm}).\\
    $\ism$ & Implicit Score Matching loss (ISM) loss defined in Eq.~(\ref{eq:ism}). \\
    $\ssm$ & Sliced Score Matching loss (SSM) loss defined in Eq.~(\ref{eq:ssm}).\\
    $\dsm$ & Denoising Score Matching loss (DSM) loss defined in Eq.~(\ref{eq:dsm}). \\
    $\total$ & the total loss of QCSBMs defined in Eq.~(\ref{eq:loss_total}).\\
    $\reg$ & the proposed regularization loss defined in Eq.~(\ref{eq:reg_qc}).\\
    $\regtrace$ & the equivalent variant of $\reg$ defined in Eq.~(\ref{eq:reg_qc_trace}).\\
    $\regest$ & the approximated variant of $\regtrace$ defined in Eq.~(\ref{eq:reg_qc_est}).\\
    $\regestt$ & the variant of $\regest$ defined in Eq.~(\ref{eq:reg_qc_estt}).\\
    \hline
    $\vu^T\vv=\vu \cdot \vv=\sum_i\vu_i \vv_i $ & inner product between two vectors $\vu$, $\vv$. \\
    $\trace{A}=\sum_i A_{i,i}$ & trace of a matrix $A$. \\
    $\norm{\vu}=\sqrt{\sum_i\vu_i^2}$ & Euclidean norm of a vector $\vu$. \\
    $\forbnorm{A}=\sqrt{\sum_{i,j} A_{i,j}^2}$ & Frobenius norm of a matrix $A$. \\
    $\exponential{\cdot}$ & an exponential function. \\
    $\stopgradop{\cdot}$ & a stop gradient operator. \\
    \boldtoprule
\end{tabularx}
\vspace{-0.5em}
\caption{The list of symbols used in this paper.}
\label{tab:notation}
\end{table*}
\subsection{Derivations}
\label{apx:derivations}
\subsubsection{Conservativeness of a Score-based Model}
\label{apx:derivations:rot}
\begin{proposition}
\label{prop:asym_cons}
Given $p(\vx)>0, \forall \vx \in \Rd$, $\asym$ equals to 0 if and only if $\curlbop{\scoremodel(\vtx;\param)} = 0,\ \forall\ 1 \leq i,j \leq \xdim$.
\end{proposition}
\begin{proof}
\begin{equation*}
\begin{aligned}
&\E_{\psigma(\vtx)}\Bigg[ \half\sum_{i,j=1}^\xdim \left(\curlop{\scoremodel(\vtx;\param)}\right)^2\Bigg] = 0 \\
\Leftrightarrow\ & \int_{\vx\in \Rd} \psigma(\vtx) \half\sum_{i,j=1}^\xdim \left(\curlop{\scoremodel(\vtx;\param)}\right)^2 d\vx = 0\\
\stackrel{(1)}{\Leftrightarrow}\  & \half\sum_{i,j=1}^\xdim \left(\curlop{\scoremodel(\vtx;\param)}\right)^2 = 0 \\
\Leftrightarrow\ & \curlbop{\scoremodel(\vtx;\param)} = 0,\ \forall\ 1 \leq i,j \leq \xdim,\\
\end{aligned}
\end{equation*}
where $(1)$ is due to the assumption of positiveness (i.e., $p(\vx)>0$). 
\end{proof}

\subsubsection{The Derivation of $\regtrace$ in Eq.~(\ref{eq:reg_qc_trace})}
\label{apx:derivations:reg_trace}
In Section~\ref{sec:methodology:regularization}, we derived the computationally efficient objective $\regestt$ based on the equality $\reg=\regtrace$. To show that the equivalence holds, we provide a formal derivation as follows.

\begin{proposition}
\label{prop:Lqc}
$\reg(\param)=\regtrace(\param)$.
\end{proposition}
\begin{proof}
\begin{equation*} 
\begin{aligned}
\reg(\param) &= \E_{\psigma(\vtx)} \left[ \half \forbnorm{\jacob - \jacob^T}^2 \right] \\
&\stackrel{(1)}{=} \E_{\psigma(\vtx)} \left[  \frac{1}{2}\trace{(\jacob - \jacob^T)^T (\jacob - \jacob^T)}\right] \\
&= \E_{\psigma(\vtx)} \left[  \frac{1}{2}\trace{(\jacob^T - \jacob) (\jacob - \jacob^T)}\right] \\
&= \E_{\psigma(\vtx)} \left[  \frac{1}{2}\trace{\jacob^T\jacob-\jacob^T\jacob^T+\jacob\jacob^T-\jacob\jacob}\right] \\
&= \E_{\psigma(\vtx)} \left[  \frac{1}{2}\left(\trace{\jacob^T\jacob}-\trace{\jacob^T\jacob^T}+\trace{\jacob\jacob^T}-\trace{\jacob\jacob}\right)\right] \\
&\stackrel{(2)}{=} \E_{\psigma(\vtx)} \left[  \trace{\jacob \jacob^T} - \trace{\jacob \jacob}\right] \\
&=\regtrace(\param),
\end{aligned}
\end{equation*}
where $(1)$ and $(2)$ are derived based on the properties of the trace operation, i.e., $\forbnorm{A}^2=\trace{A^TA}$ and $\trace{AB}=\trace{BA}$, respectively.
\end{proof}

\subsection{A Detailed Description of the Training Process of QCSBMs}
\label{apx:optimization}
The entire training procedure is divided into five steps, denoted as Steps (1)$\sim$(5), respectively. Steps (1)$\sim$(3) describe the forward propagation process of $\total(\param)$, which is depicted by the black arrows in Fig.~\ref{fig:bp}~(a). 
Steps (4) and (5) correspond to the backpropagation processes of the two gradient components comprising $\gradop{\param}\total(\param)$, which are depicted in Fig.~\ref{fig:bp}~(b). In the following paragraphs, we elaborate on the details of Steps (1)$\sim$(5).

\begin{figure}[t]
    \centering
    \includegraphics[width=0.9\linewidth]{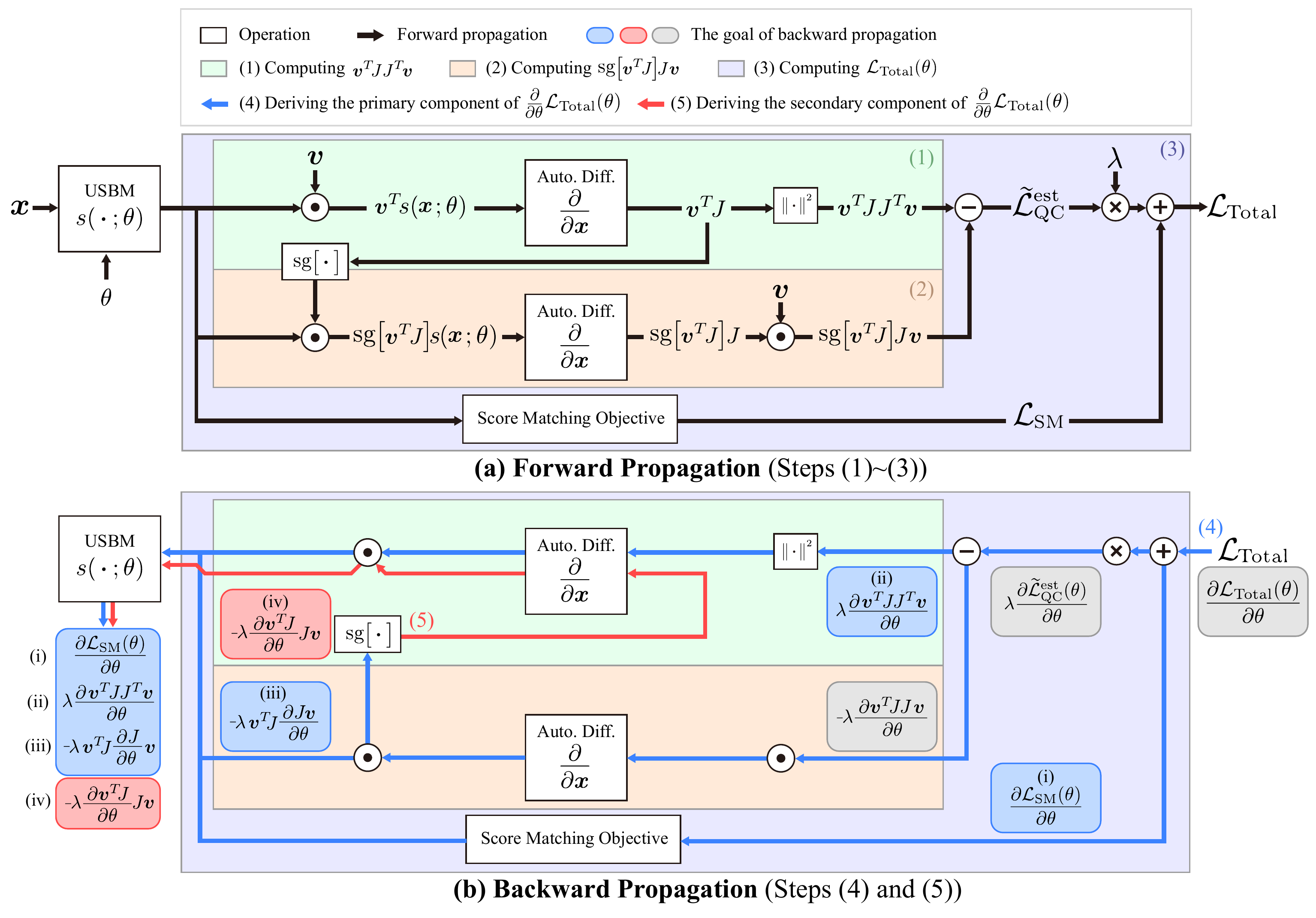}
    \vspace{-1em}
    \caption{The computational graphs of $\total$ in QCSBMs. The upper and lower subplots depict the forward and backward propagation processes, respectively. The `Auto. Diff.' blocks represent the operation of differentiating $\vu^T \scoremodel(\vtx;\param)$, where $\vu$ is a constant vector with respect to $\vtx$.}
    \label{fig:bp}
\end{figure}
\textbf{(1) Computing $\vv^T \jacob \jacob^T \vv$.}~First, $\vv^T \jacob$ is computed by performing backpropagation of $\vv^T \scoremodel(\vtx;\param)$ with respect to $\vtx$ via automatic differentiation, which is depicted as the upper `Auto. Diff.' block in Fig.~\ref{fig:bp}~(a). Then, $\vv^T \jacob \jacob^T \vv$ is calculated by taking the squared L2 norm on $\vv^T \jacob$ according to the relationship: $\norm{\vv^T \jacob}^2=\vv^T \jacob (\vv^T \jacob)^T=\vv^T \jacob \jacob^T \vv$.

\textbf{(2) Computing $\vv^T \jacob \jacob \vv$.}~$\stopgradop{\vv^T \jacob} \scoremodel(\vtx;\param)$ is first calculated by taking the inner product between $\stopgradop{\vv^T \jacob}$ and $\scoremodel(\vtx;\param)$, where the stop-gradient operator $\stopgradop{\cdot}$ is applied to $\vv^T \jacob$ to detach it from the computational graph built in Step (1). Then, $\stopgradop{\vv^T \jacob} \jacob$ is calculated by differentiating $\stopgradop{\vv^T\jacob}\scoremodel(\vtx;\param)$ via performing backpropagation. Stopping the gradient of $\vv^T\jacob$ is necessary to ensure that the automatic differentiation (i.e., the lower `Auto. Diff.' block in Fig.~\ref{fig:bp}~(a)) excludes the computational graph used for differentiating $\vv^T\jacob$, allowing $\vv^T\jacob \jacob$ to be correctly derived. Lastly, $\stopgradop{\vv^T \jacob} \jacob \vv$ is obtained by taking the inner product of $\stopgradop{\vv^T\jacob}\jacob$ and $\vv$.

\textbf{(3) Computing $\total(\param)$.}~Based on the results of Steps (1) and (2), $\regestt(\param)$ is computed by taking the expectation of ($\vv^T \jacob \jacob^T \vv -\stopgradop{\vv^T \jacob}\jacob\vv$). Meanwhile, the score matching loss $\sm(\param)$ can be derived using any one of the Eqs.~(\ref{eq:esm}),~(\ref{eq:ism}),~(\ref{eq:ssm}), and~(\ref{eq:dsm}). Finally, $\total(\param)$ is calculated by adding $\sm(\param)$ and $\lambda\regestt(\param)$, as described in Eq.~(\ref{eq:loss_total}).

\textbf{(4) Deriving the primary component of $\gradop{\param}\total(\param)$.}~Based on the computational graph built in Steps (1)$\sim$(3), the primary component of $\gradop{\param}\total(\param)$ is computed by performing backward propagation through the paths in the computational graph highlighted by the blue arrows in Fig.~\ref{fig:bp}~(b) using automatic differentiation. Note that these gradients are not equal to $\gradop{\param}\total(\param)$ due to the adoption of the stop-gradient operator $\stopgradop{\cdot}$ in Step (2). As a result, an additional secondary gradient component, which is derived in Step (5), is included to compensate it.

\textbf{(5) Deriving the secondary component of $\gradop{\param}\total(\param)$.}~The secondary component of $\gradop{\param}\total(\param)$ is derived by performing backward propagation through the paths in the computational graph highlighted by the red arrows in Fig.~\ref{fig:bp}~(b) using automatic differentiation. By accumulating the gradients of the primary and the secondary components, the gradients $\gradop{\param}\total(\param)$ can be correctly calculated.

\textbf{The Derivation of the primary and secondary components of $\gradop{\param} \total(\param)$.}~In Steps (4) and (5), we decompose $\gradop{\param}\total(\param)$ as the primary and secondary components, and separately derive them. To further elaborate on such a backward propagation process, we offer a detailed description in this subsection. For the sake of notational simplicity, we assume that both the batch size and the number of random vectors $\estdim$ are 1.

According to the rule of sum and the rule of product from vector calculus, the gradient of the total loss $\gradop{\param}\total(\param)$ can be decomposed as the sum of $\gradop{\param} \sm(\param)$, $\lambda \gradop{\param} \vv^T \jacob \jacob^T \vv$, $-\lambda (\gradop{\param} \vv^T\jacob) \jacob\vv$, and $-\lambda \vv^T\jacob (\gradop{\param} \jacob\vv)$, indexed as (i)$\sim$(iv) respectively. The derivation is shown as the following:
\begin{equation*} 
\label{eq:chain_rule}
\begin{aligned}
\grad{\total(\param)}{\param}&=\grad{(\sm(\param)+\lambda\regest(\param))}{\param}\\
&=\grad{\sm(\param)}{\param}+\lambda\grad{\regest(\param)}{\param}\\
&=\grad{\sm(\param)}{\param}+\lambda\grad{(\vv^T\jacob\jacob^T\vv-\vv^T\jacob\jacob\vv)}{\param}\\
&=\grad{\sm(\param)}{\param}+\lambda\grad{\vv^T\jacob\jacob^T\vv}{\param}-\lambda\grad{\vv^T\jacob\jacob\vv}{\param}\\
&=\underbrace{\grad{\sm(\param)}{\param}}_\text{(i)}+\underbrace{\lambda\grad{\vv^T\jacob\jacob^T\vv}{\param}}_\text{(ii)}+\underbrace{(-\lambda)\vv^T\jacob\grad{\jacob\vv}{\param}}_\text{(iii)}+\underbrace{(-\lambda)\grad{\vv^T\jacob}{\param}\jacob\vv}_\text{(iv)}.
\end{aligned}
\end{equation*}

We name the sum of (i)$\sim$(iii) the \textit{primary} component of $\gradop{\param}\total(\param)$, and the term (iv) the \textit{secondary} component of $\gradop{\param}\total(\param)$. Such a decomposition suggests that $\gradop{\param} \total(\param)$ can be separately computed based on the computational graph built in Steps (1)$\sim$(3) as shown in the upper subplot of Fig.~\ref{fig:bp}. For the primary component, the sum of (i)$\sim$(iii) is computed by performing backward propagation through the paths in the computational graph highlighted by the blue arrows in the lower subplot of Fig.~\ref{fig:bp} using automatic differentiation. For the secondary component, the term (iv) is calculated by performing backward propagation through the red arrows in the lower subplot of Fig.~\ref{fig:bp}. Through these two steps, $\gradop{\param}\total(\param)$ can be correctly derived.

\subsection{Normalized Asymmetry Metric}
\label{apx:asymmetry}
In this section, we elaborate on the formulation of the normalized asymmetry metric $\normasym$ and derive a computationally efficient implementation of it using the Hutchinson's trace estimator.

\textbf{Derivation of the $\normasym$ Metric.}~As described in~\citep{ANDRILLI20161}, any matrix $A$ can be uniquely decomposed into a symmetric matrix $A_\mathrm{sym}$ and a skew-symmetric matrix $A_\mathrm{skew}$ as follows:
\begin{equation} 
\label{eq:skew_matrix}
A=A_\mathrm{sym}+A_\mathrm{skew}=\frac{A+A^T}{2}+\frac{A-A^T}{2}.
\end{equation}
Based on Eq.~(\ref{eq:skew_matrix}), the Jacobian $\jacob$ of a USBM $\scoremodel(\cdot\,;\param)$ can be written as the sum of a symmetric matrix $\jacob_\mathrm{sym}=(\jacob+\jacob^T)/2$ and a skew-symmetric matrix $\jacob_\mathrm{skew}=(\jacob-\jacob^T)/2$. Under such a definition, the $\normasym$ metric introduced in Section~\ref{sec:motivational_example:curl} can be formulated as follows:
\begin{equation} 
\label{eq:asym_ratio}
\E_{\psigma(\vtx)}\left[\frac{\forbnorm{\jacob_\mathrm{skew}}^2}{\forbnorm{\jacob}^2}\right]
=\E_{\psigma(\vtx)}\left[\frac{\forbnorm{\half(\jacob - \jacob^T)}^2}{\forbnorm{\jacob}^2}\right]
=\E_{\psigma(\vtx)}\left[\frac{1}{4}\frac{\forbnorm{\jacob - \jacob^T}^2}{\forbnorm{\jacob}^2}\right].
\end{equation}
This metric measures the ratio of the squared Frobenius norm of the skew-symmetric matrix $\forbnorm{\jacob_\mathrm{skew}}^2$ to the squared Frobenius norm of the Jacobian matrix $\forbnorm{\jacob}^2$, and falls within the range $[0,1]$. $\normasym=1$ corresponds to the condition where $\jacob_\mathrm{skew}$ dominates $\jacob$, implying that $\jacob$ is skew-symmetric. On the contrary, $\normasym=0$ indicates that $\jacob$ only contains the symmetric component $\jacob_\mathrm{sym}$, suggesting that $\jacob$ is symmetric. Since the squared Frobenius norm of the skew-symmetric matrix can be written as the sum of the squared rotation densities of $\scoremodel(\vtx;\param)$, i.e., $\forbnorm{\jacob_\mathrm{skew}}^2=\forbnorm{(\jacob-\jacob^T)/2}^2=\frac{1}{4} \sum_{i,j=1}^\xdim (\gradop{\vtx_j} \scoremodel(\vtx;\param)_i - \gradop{\vtx_i} \scoremodel(\vtx;\param)_j)^2=\frac{1}{4} \sum_{i,j=1}^\xdim \left(\curlop{\scoremodel(\vtx;\param)}
\right)^2$, $\normasym$ can be adopted to measure the non-conservativeness of $\scoremodel(\cdot\,;\param)$, as mentioned in Section~\ref{sec:motivational_example:curl}. 

\textbf{An Efficient Implementation of $\normasym$.}~Since Eq.~(\ref{eq:asym_ratio}) involves the explicit calculation of the Jacobian matrix $\jacob$, evaluating the $\normasym$ metric for any single instance requires $\xdim$ times of backward propagations. This indicates that the evaluation cost could grow significantly when $\xdim$ becomes large. To reduce the evaluation cost, we utilize the Hutchinson's trace estimator to approximate the $\normasym$ metric based on the following derivation:
\begin{equation} 
\label{eq:asym_norm_est}
\begin{aligned}
\E_{\psigma(\vtx)}\left[\frac{1}{4}\frac{\forbnorm{\jacob - \jacob^T}^2}{\forbnorm{\jacob}^2}\right]&=\E_{\psigma(\vtx)}\left[\frac{1}{2}\frac{\trace{\jacob \jacob^T}-\trace{\jacob \jacob}}{\trace{\jacob \jacob^T}}\right] \\ 
&=\E_{\psigma(\vtx)}\left[\frac{1}{2}\frac{\E_{\pv(\vv)}[\vv^T\jacob \jacob^T \vv]-\E_{\pv(\vv)}[\vv^T\jacob \jacob\vv]}{\E_{\pv(\vv)}[\vv^T\jacob \jacob^T \vv]}\right] =\E_{\psigma(\vtx)\pv(\vv)}\left[\frac{1}{2}\frac{\vv^T\jacob \jacob^T \vv-\vv^T\jacob \jacob\vv}{\vv^T\jacob \jacob^T \vv}\right].
\end{aligned}
\end{equation}
The expectation $\E_{\pv(\vv)}[\cdot]$ can be approximated using $\estdim$ random vectors. In addition, the terms $\vv^T\jacob \jacob^T \vv$ and $\vv^T\jacob \jacob\vv$ in Eq.~(\ref{eq:asym_norm_est}) can be efficiently calculated based on Steps (1) and (2) described in Section~\ref{sec:methodology:implementation}. This suggests that the computational cost of evaluating $\normasym$ can be significantly reduced when $\estdim \ll \xdim$.
\subsection{Time-Inhomogeneous QCSBMs}
\label{apx:sampling}
In this section, we demonstrate how a QCSBM is converted to its time-inhomogeneous variant QC-NCSN++, which was described in Section~\ref{sec:experimental_results} of the main manuscript. We first explain the modifications made in the sampling process. Then, we elaborate on the corresponding adjustments in the score-matching objective and the regularization loss.

\textbf{Sampling Process.}~QC-NCSN++ adopts the variance exploding (VE) diffusion process identical to that employed in NCSN++~(VE)~\citep{song2021scorebased}, which is a time-inhomogeneous sampling algorithm. In this sampling algorithm, the SBM and the step size are respectively represented as $\scoremodel(\cdot \ ;\param,\sigma_t)$ and $\alpha_t=\gradop{t}\sigma_t^2$, where $\sigma_t$ is a time-dependent standard deviation. In C-NCSN++, U-NCSN++, and QC-NCSN++, $\sigma_t$ is set to $\sigma_\mathrm{min}(\sigma_\mathrm{min}/\sigma_\mathrm{max})^{\frac{t}{T}}$~\citep{song2021scorebased}, where $T$ is the total number of timesteps in the sampling process, $\sigma_\mathrm{min}$ is a constant representing a minimal noise scale, and $\sigma_\mathrm{max}$ is a constant denoting a maximal noise scale.

\textbf{Training Objectives.}~Since the above time-inhomogeneous sampling process requires the SBM $\scoremodel(\cdot \ ;\param,\sigma_t)$ to be conditioned on a time-dependent standard deviation $\sigma_t$, the training objectives of $\scoremodel(\cdot \ ;\param,\sigma_t)$ have to be modified accordingly. For example, the score-matching objective $\dsm$ used in C-NCSN++, U-NCSN++, and QC-NCSN++ is modified as follows:
\begin{equation} 
\label{eq:dsm_noise}
\E_{\mathcal{U}(t)}  \left[ \lambda(t) \E_{p_{\sigma_t}(\vttx|\vx)\pzero(\vx)} \left[\norm{\scoremodel(\vttx;\param,\sigma_t) - \grad{\log p_{\sigma_t}(\vttx|\vx)}{\vttx}}^2 \right]\right],
\end{equation}

where $\mathcal{U}(t)$ is a uniform distribution defined on the interval $[0,T]$, and $\lambda(t)$ is a time-dependent coefficient for balancing the loss functions of different $t$. Meanwhile, the regularization term $\regestt$ used in QC-NCSN++ is adjusted according to $\lambda(t)$, which is formulated as follows:
\begin{equation} 
\label{eq:reg_qc_est_noise}
\begin{aligned}
\E_{\mathcal{U}(t)}  \left[ \lambda(t) \E_{p_{\sigma_t}(\vttx|\vx)\pzero(\vx)} \left[ \E_{\pv(\vv)} \left[\vv^T \jacob \jacob^T \vv -\vv^T \jacob \jacob \vv \right]\right] \right],
\end{aligned}
\end{equation}
where $\jacob = \gradop{\vttx} \scoremodel(\vttx;\param,\sigma_t)$.
\subsection{Experimental Setups}
\label{apx:configuration}
In this section, we elaborate on the experimental configurations and provide the detailed hyperparameter setups for the experiments presented in Sections~\ref{sec:motivational_example} and \ref{sec:experimental_results} of the main manuscript. The code implementation for the experiments is provided in the following repository: \url{https://github.com/chen-hao-chao/qcsbm}.

\subsubsection{Experimental Setups for Section~\ref{sec:motivational_example:curl}}
\label{apx:configuration:curl}
In Section~\ref{sec:motivational_example:curl}, we compare the sampling efficiency of a USBM and a CSBM with an approximation error $\epsilon$. These SBMs are formulated based on the following equation:
\begin{equation} 
\label{eq:scoremodel_epsilon}
\scoremodel(\vtx) =\frac{\partial}{\partial \vtx}\log \psigma(\vtx) +\sqrt{\frac{2\epsilon \mu(\vtx)}{\norm{\vtx}^2\psigma(\vtx)}}\vu(\vtx),
\end{equation}

where $\psigma$ is the target distribution, $\mu$ is an arbitrary distribution, and $\vu(\vtx)\in \Rd$ is a vector function with its norm equal to the norm of its input (i.e., $\|\vu(\vtx)\|=\|\vtx\|$). To show that the SBM $\scoremodel(\cdot)$ in Eq.~(\ref{eq:scoremodel_epsilon}) satisfies $\esm= \epsilon$, we provide the following proposition.

\begin{proposition}
Given $\epsilon>0$, a target distribution $\psigma$, and an arbitrary pdf $\mu$, $\scoremodel$ defined in Eq.~(\ref{eq:scoremodel_epsilon}) satisfies $\esm=\epsilon$.
\end{proposition}
\begin{proof}
\begin{equation*} 
\begin{aligned}
\esm &=\int_{\vtx} \psigma(\vtx)\half \norm{\scoremodel(\vtx)-\gradop{\vtx}\log \psigma(\vtx)}^2 d\vtx\\
     &=\int_{\vtx} \psigma(\vtx) \half \norm{\frac{\partial}{\partial \vtx} \log \psigma(\vtx) +\sqrt{\frac{2\epsilon \mu(\vtx)}{\norm{\vtx}^2\psigma(\vtx)}}\vu(\vtx)-\gradop{\vtx}\log \psigma(\vtx)}^2 d\vtx\\
     &=\int_{\vtx} \psigma(\vtx) \half \norm{\sqrt{\frac{2\epsilon \mu(\vtx)}{\norm{\vtx}^2\psigma(\vtx)}}\vu(\vtx)}^2 d\vtx =\int_{\vtx} \psigma(\vtx) \half \frac{2\epsilon \mu(\vtx)}{\norm{\vtx}^2\psigma(\vtx)} \norm{\vu(\vtx)}^2 d\vtx \\
     & = \int_{\vtx} \mu(\vtx) \frac{\epsilon \norm{\vu(\vtx)}^2}{\norm{\vtx}^2}  d\vtx=\int_{\vtx} \mu(\vtx)\epsilon d\vtx=\epsilon\\
\end{aligned}
\end{equation*}
\end{proof}
In the motivational example presented in Section~\ref{sec:motivational_example:curl}, we choose $\psigma=\N(0;\sigma^2I)$, and select $\mu=\frac{1}{10}\sum_{i=1}^{10} \N([3\cos(\frac{2i\pi}{10}),3\sin(\frac{2i\pi}{10})]^T;I)$. We consider $\vu(\vtx)=[-\tx_2, \tx_1]^T$ for $\scoremodel_{\mathrm{U}}$, and $\vu(\vtx)=[\tx_1, \tx_2]^T$ for $\scoremodel_{\mathrm{C}}$, where $\vtx=[\tx_1,\tx_2]^T$. In particular, $[-\tx_2, \tx_1]^T$ is a rotational vector field with each vector tangent to the true score function $\frac{\partial}{\partial \vtx}\log \psigma(\vtx)=-1/\sigma^2 [\tx_1, \tx_2]^T$. On the other hand, $[\tx_1, \tx_2]^T$ is a vector field with each vector pointing to the opposite direction against the true score function. For an illustrative purpose, we leverage a deterministic variant of Eq.~(\ref{eq:langevin}) (i.e., $\vtx_{t+1}=\vtx_t+\alpha_t\scoremodel(\vtx_t)$) as our sampler to generate samples based on $\scoremodel_{\mathrm{U}}$ and $\scoremodel_{\mathrm{C}}$, and calculate the steps required for all samples to move to the center of $\psigma$.

\subsubsection{Experimental Setups for the Motivational Examples}
\label{apx:configuration:motivational_example}
\begin{wrapfigure}{r}{0.4\linewidth}
\vspace{-1.4em}
        \centering
    \includegraphics[width=\linewidth]{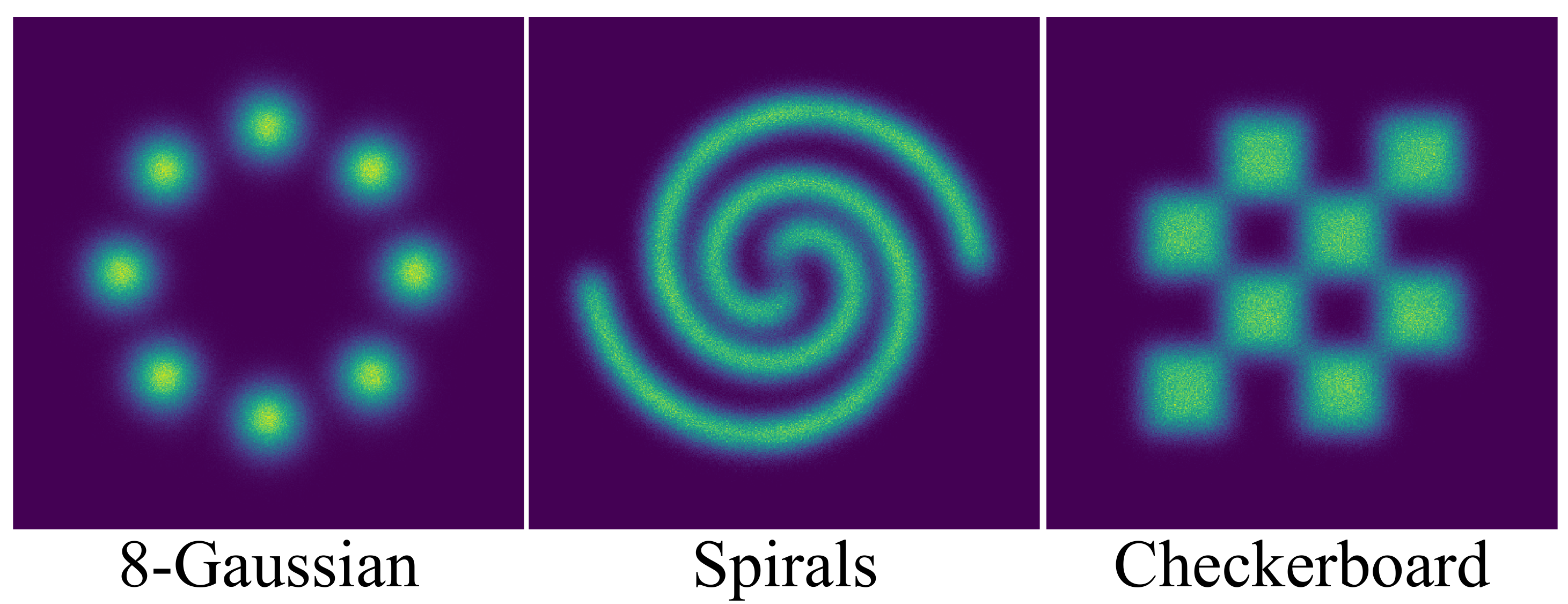}
    \vspace{-2em}
    \caption{The visualizations of the 8-Gaussian, Spirals, and Checkerboard datasets.}
    \label{fig:toy_datasets}
        \vspace{-1em}
\end{wrapfigure} 
\textbf{Datasets.}~The motivational experiments in Section~\ref{sec:motivational_example:flexibility} are performed on the 8-Gaussian, Spirals, and Checkerboard datasets as shown in Fig.~\ref{fig:toy_datasets}~(a). The data points of the 8-Gaussian dataset are sampled from eight separate Gaussian distributions centered at $(\cos(\frac{\pi w}{4}), \sin(\frac{\pi w}{4}))$, where $w\in \{1,...,8\}$. The data points of the Spirals dataset are sampled from two separate curves $(-\pi\sqrt{w}\cos(\pi\sqrt{w}), \pi\sqrt{w}\sin(\pi\sqrt{w}))$ and $(\pi\sqrt{w}\cos(\pi\sqrt{w}), -\pi\sqrt{w}\sin(\pi\sqrt{w}))$, where $w\in [0,1]$. Lastly, the data points of the Checkerboard dataset are sampled from $(4w-2,t-2s+\lfloor 4w-2\rfloor \,\mathrm{mod}\,2)$, where $w\in[0,1]$, $t\in[0,1]$, $s\in\{0,1\}$, $\lfloor\cdot\rfloor$ is a floor function, and $\mathrm{mod}$ is the modulo operation.

\textbf{Training and Implementation Details.}~The network architecture of $f$ is a three-layered multilayer perceptron (MLP) with $(64,128,256)$ neurons and Swish~\citep{ramachandran2017searching} as its activation function. This model architecture is similar to that used in the two-dimensional experiments of \citep{chao2022denoising}. The SBMs $\scoremodel_\mathrm{U}$ and $\scoremodel_\mathrm{C}$ are trained using the Adam optimizer~\citep{kingma2014adam} with a learning rate of $7.5\times 10^{-4}$ and a batch size of $5,000$. The balancing factor $\lambda$ is fixed to $0.1$. The maximal and minimal noise scales $\sigma_\mathrm{max}$ and $\sigma_\mathrm{min}$ are set to $3$ and $0.1$, respectively.

\textbf{Evaluation Method.}~The asymmetry ($\asym$) and normalized asymmetry ($\normasym$) metrics are evaluated on $\timedim$ discretized timesteps according to the following equations:
\begin{equation} 
\label{eq:asym_toy}
\frac{1}{\timedim\datasetdim}\sum_{t=1}^{\timedim} \sum_{m=1}^{\datasetdim} \trace{ \jacob \jacob^T}  - \trace{\jacob \jacob},
\end{equation}
\begin{equation} 
\label{eq:asym_norm_toy}
\frac{1}{\timedim\datasetdim}\sum_{t=1}^{\timedim} \sum_{m=1}^{\datasetdim} \frac{ \trace{\jacob \jacob^T} -\trace{\jacob \jacob}}{\trace{\jacob \jacob^T}},
\end{equation}
where $\jacob = \gradop{\tilde{\rvx}^{(t)}_m} \scoremodel(\tilde{\rvx}^{(t)}_m;\param,\sigma_t)$, $\{\tilde{\rvx}^{(t)}_m\}_{m=1}^\datasetdim$ represents a set of testing data points, and $\datasetdim$ denotes the size of the testing set. On the other hand, the score error is evaluated based on the following formula:
\begin{equation} 
\label{eq:score_err}
\frac{1}{\timedim\datasetdim}\sum_{t=1}^{\timedim} \sum_{m=1}^{\datasetdim} \norm{\scoremodel(\tilde{\rvx}^{(t)}_m;\param,\sigma_t) - \grad{\log p_{\sigma_t}(\tilde{\rvx}^{(t)}_m)}{\tilde{\rvx}^{(t)}_m}}^2,
\end{equation}
where $p_{\sigma_t}(\vttx)=\int p_{\sigma_t}(\vttx|\vx) p_0(\vx)d \vx$ and its closed form is derived according to \citep{chao2022denoising}. In our implementation, $\timedim$ and $\datasetdim$ are set as $10$ and $5,000$, respectively.

\vspace{1em}
\subsubsection{Experimental Setups for the Evaluations on the Real-World Datasets}
\label{apx:configuration:datasets}

\textbf{Datasets.}~The experiments presented in Section~\ref{sec:experimental_results} are performed on the CIFAR-10, CIFAR-100~\citep{krizhevsky2009learning}, ImageNet-32x32~\citep{van2016pixel}, and SVHN~\citep{Netzer2011} datasets. The training and test sets of Cifar-10 and Cifar-100 contain 50,000 and 10,000 images, respectively. The training and test sets of SVHN contain 73,257 and 26,032 images, respectively. On the other hand, the training and the test sets of ImageNet-32x32 consist of 1,281,149 and 49,999 images, respectively. 

\textbf{Training and Implementation Details.}~C-NCSN++, U-NCSN++, and QC-NCSN++ are implemented using the \texttt{Pytorch} framework. C-NCSN++, U-NCSN++, and QC-NCSN++ are trained using the Adam optimizer with a learning rate of $2\times10^{-4}$. The training procedure of U-NCSN++ requires 600,000 iterations for convergence for the CIFAR-10, CIFAR-100, and ImageNet-32x32 datasets, while it requires 300,000 iterations for convergence for the SVHN dataset. On the other hand, the optimization of C-NCSN++ is terminated at the early stage of the training process, since we observed that the sampling process using ODE sampler of C-NCSN++ optimized according to the aforementioned training length fails to converge. To address this issue, in our experiments, the training iterations of C-NCSN++ are reduced to 450,000 for the CIFAR-10 and ImageNet-32x32 datasets, 300,000 for the CIFAR-100 dataset, and 100,000 for the SVHN dataset. The training procedure of QC-NCSN++ consists of two stages. In the first stage, QC-NCSN++ is optimized in the same manner as U-NCSN++. In the second stage, the regularization term $\regestt$ is incorporated during the training process, which requires additional 150,000 iterations for convergence. In the training process, the batch size $b$ is fixed to 128 for C-NCSN++ and U-NCSN++, while its value is adjusted according to $\estdim$ for QC-NCSN++ for conserving the memory consumption. We adopt $(b,\estdim)=(8,16)$ in Table~\ref{tab:nll}, and $(b,\estdim)=(128,1)$ in Tables~\ref{tab:fidis_pr_nfe} and \ref{tab:pc} for QC-NCSN++, respectively. The maximal and minimal noise scales $\sigma_\mathrm{max}$ and $\sigma_\mathrm{min}$ are set to 50 and 0.01, respectively. The balancing factor $\lambda$ is set to 0.0001. The ODE sampler is implemented using the \texttt{scipy.integrate.solve\_ivp} library. 

\textbf{Evaluation Method.}~The asymmetry ($\asym$) and normalized asymmetry ($\normasym$) metrics are evaluated using Eqs.~(\ref{eq:asym_realworld}) and (\ref{eq:asym_norm_realworld}), respectively. They are formulated as follows:
\begin{equation} 
\label{eq:asym_realworld}
\frac{1}{\timedim\datasetdim \estdim}\sum_{t=1}^{\timedim} \sum_{m=1}^{\datasetdim} \sum_{k=1}^{\estdim} \rvv_k^T \jacob \jacob^T \rvv_k -\rvv_k^T \jacob \jacob \rvv_k,
\end{equation}
\begin{equation} 
\label{eq:asym_norm_realworld}
\frac{1}{\timedim\datasetdim \estdim}\sum_{t=1}^{\timedim} \sum_{m=1}^{\datasetdim} \sum_{k=1}^{\estdim} \frac{\rvv_k^T \jacob \jacob^T \rvv_k -\rvv_k^T \jacob \jacob \rvv_k}{\rvv_k^T \jacob \jacob^T \rvv_k},
\end{equation}
where $\jacob = \gradop{\tilde{\rvx}^{(t)}_m} \scoremodel(\tilde{\rvx}^{(t)}_m;\param,\sigma_t)$, $\{\tilde{\rvx}^{(t)}_m\}_{m=1}^\datasetdim$ represents a set of testing data points, and $\{\rvv_k\}_{k=1}^\estdim$ is a set of i.i.d. samples drawn from $\pv$. In our implementation, $\timedim$ and $\estdim$ are set as $100$ and $1$, respectively. The metrics for sampling performance (i.e., FID, IS, Precision and Recall) are evaluated using the \texttt{tensorflow\_gan} library as well as the official evaluation package implemented by~\citep{Kynknniemi2019ImprovedPA, ferjad2020icml}. 

\begin{wraptable}{r}{0.45\linewidth}
\renewcommand{\arraystretch}{0.9}
    \newcommand{\boldtoprule}{\midrule[1.1pt]}
    \newcommand{\thline}{\midrule[0.3pt]}
    \centering
    \vspace{-2.2em}
    \caption{The confidence interval of the evaluation results on the CIFAR-10 dataset.}
    \vspace{0.2em}
    \resizebox{\linewidth}{!}{%
    \begin{tabular}{ccccc}
        \boldtoprule
        NLL& FID & IS & Precision & Recall \\
        \thline
        $\pm$0.0014 & $\pm$0.0143 & $\pm$0.0065 & $\pm$0.0024 & $\pm$0.0011\\
        \boldtoprule
    \end{tabular}
    }
    \vspace{-1.5em}
    \label{tab:ci}
\end{wraptable} 
\textbf{Confidence Intervals of the Evaluation Results.}~Table~\ref{tab:ci} shows the 95\% confidence intervals for the evaluation results of QC-NCSN++ in terms of the NLL, FID, IS, Precision, and Recall metrics on the CIFAR-10 dataset. For the evaluation results of the FID, IS, Precision, and Recall metrics, the PC sampler with NFE=1,000 is adopted. All of these results are obtained by three times of evaluations.

\begin{wraptable}{r}{0.4\linewidth}
\renewcommand{\arraystretch}{0.9}
    \newcommand{\boldtoprule}{\midrule[1.1pt]}
    \newcommand{\thline}{\midrule[0.3pt]}
    \centering
    \vspace{-1em}
    \caption{A comparison between the results reported in~\citep{Xu2022PoissonFG} and those reproduced by us for U-NCSN++.}
    \vspace{0.2em}
    \resizebox{\linewidth}{!}{
    \begin{tabular}{ccccc}
        \boldtoprule
        & FID & IS & NFE \\
        \thline
        U-NCSN++~\citep{Xu2022PoissonFG} & 7.66 & 9.17 & 194 \\
        U-NCSN++ (Ours) & 7.48 & 9.24 & 170 \\
        \thline
        QC-NCSN++ (Ours) & \textbf{7.21} & \textbf{9.25} & \textbf{124} \\
        \boldtoprule
    \end{tabular}
    }
    \vspace{-1em}
    \label{tab:reproduce}
\end{wraptable} 
\textbf{Sampling performance of U-NCSN++.}~Table~\ref{tab:reproduce} compares the sampling performance of the baseline method (i.e., U-NCSN++) reported in~\citep{Xu2022PoissonFG} and that reproduced by us on the CIFAR-10 dataset using an ODE sampler. It is observed that the reproduced results are improved in terms of the FID, IS, and NFE metrics. This reinforces our statement in Section~\ref{sec:experimental_results}, as QC-NCSN++ is able to achieve superior results to both the reproduced and reported performance of U-NCSN++.

\subsection{Additional Experimental Results}
\label{apx:experiments}
In this section, we provide a number of additional experimental results. In Section~\ref{apx:experiments:mechanism}, we present additional experimental results of QCSBMs implemented as one-layered autoencoders to support our observation presented in Section~\ref{sec:mechanism} of the main manuscript. In Section~\ref{apx:experiments:time}, we provide a comparison between C-NCSN++, U-NCSN++, and QC-NCSN++ in terms of their time and memory consumption for each training and sampling iteration. In Section~\ref{apx:experiments:lambda}, we demonstrate the impact of the choices of $\lambda$ on the performance of QC-NCSN++. Finally, in Section~\ref{apx:experiments:uncurated_examples}, we provide additional qualitative results on the real-world datasets.

\subsubsection{QCSBMs implemented as One-Layered Autoencoders}
\label{apx:experiments:mechanism}
In Section~\ref{sec:mechanism}, we leveraged the example of an one-layered autoencoder $\scoremodel(\vtx;\param)=\weightr h(\weightw^T\vtx+\vb)+\vc$ to demonstrate the advantage of QCSBMs over CSBMs. Our experimental results in Fig.~\ref{fig:sym_exp} reveals that QCSBMs can learn to output conservative vector fields, which cannot be captured by CSBMs with tied weights (i.e., $\weightr=\weightw$). To further solidify our empirical observation, we provide additional examples in Fig.~\ref{fig:sym_exp+}. Fig.~\ref{fig:sym_exp+} depicts the trends of $\forbnorm{\weightw\weightr^T-\weightr\weightw^T}$ and $\forbnorm{\weightw-\weightr}$ during the minimization process of $\reg$ with four different seeds. As the training progresses, $\reg$ and $\forbnorm{\weightw\weightr^T-\weightr\weightw^T}$ both approach to zero in all of these four examples. In contrast, the values of $\forbnorm{\weightw-\weightr}$ do not approach to zero, and the trends of $\forbnorm{\weightw-\weightr}$ for these four examples differ. The above experimental evidences demonstrate that QCSBMs can learn to output conservative vector fields with $\weightr\neq\weightw$, and thus justify the advantage of QCSBMs over CSBMs.

\begin{table}[t]
    \renewcommand{\arraystretch}{1.02}
    \newcommand{\boldtoprule}{\midrule[0.8pt]}
    \newcommand{\thline}{\midrule[0.3pt]}
    \centering
    \caption{The time and memory consumption of evaluating the score function, the objective function, and the gradient of the objective function of C-NCSN++, U-NCSN++, and QC-NCSN++.}
    \vspace{0.5em}
    \footnotesize
    \resizebox{\linewidth}{!}{
    \begin{tabular}{ccc||cr|cr|cr}
        \boldtoprule
        & & & \multicolumn{2}{c|}{Evaluating $\loss(\param)$} & \multicolumn{2}{c|}{Evaluating $\frac{\partial}{\partial \param}\loss(\param)$} & \multicolumn{2}{c}{Evaluating $\scoremodel(\vx;\param)$} \\
        \thline
        Method & Model ($\scoremodel(\cdot\,;\param)$) & Objective ($\loss(\param)$) & Time & Memory & Time & Memory & Time & Memory \\
        \thline
        C-NCSN++  & Conservative NCSN++ & $\sm$    & 0.26 s & 17.2 GB & 1.21 s & 19.3 GB & 0.25 s & 17.2 GB \\
        U-NCSN++  & NCSN++              & $\sm$    & 0.10 s & 8.5 GB  & 0.30 s & 8.9  GB & 0.10 s & 5.5  GB \\
        QC-NCSN++ & NCSN++              & $\total$ & 0.36 s & 27.6 GB & 2.82 s & 32.0 GB & 0.10 s & 5.5  GB \\
        \boldtoprule
    \end{tabular}
    }
    \label{tab:time_mem}
\end{table}

\subsubsection{A Comparison on the Time and Memory Consumption}
\label{apx:experiments:time}
In this section, we investigate the time and memory consumption of evaluating the score function, the objective function, and the gradient of the objective function of C-NCSN++, U-NCSN++, and QC-NCSN++. The gradient operations with respect to both of the input $\vx$ and the parameters $\param$ are implemented using the automatic differentiation tool~\citep{Griewank2000EvaluatingD} provided in the \texttt{Pytorch} library~\citep{Paszke2019PyTorchAI}. The results are evaluated on a single NVIDIA V100 GPU with $32$ GB memory, and the batch size is fixed at $32$. Table~\ref{tab:time_mem} reports the evaluation results of the above setting. It is observed that the training time and memory requirements of C-NCSN++ and QC-NCSN++ are higher than U-NCSN++ as the calculation of the objective $\loss(\param)$ and its gradients $\frac{\partial}{\partial \param}\loss(\param)$ requires additional backward propagation. On the other hand, the sampling time and memory requirements of U-NCSN++ and QC-NCSN++ are lower than C-NCSN++, since the gradient operation in evaluating $s(\cdot\,;\param)$ of C-NCSN++ is prevented.

\vspace{2em}
\begin{wraptable}{r}{0.45\linewidth}
\renewcommand{\arraystretch}{0.9}
    \newcommand{\boldtoprule}{\midrule[1.1pt]}
    \newcommand{\thline}{\midrule[0.3pt]}
    \centering
    \vspace{-2em}
    \caption{The evaluation results of QC-NCSN++ with different choices of $\lambda$ on the CIFAR-10 dataset.}
    \vspace{0.2em}
    \resizebox{\linewidth}{!}{
    \begin{tabular}{c|cccc|c}
        \boldtoprule
        $\lambda$    &  0.001 & 0.0005 & 0.0001 & 0.00005 & 0.0\\
        \thline
        $\asym$      & \textbf{9.99 e6} & 2.38 e7 & 5.03 e7 & 6.42 e7 & 1.88 e8\\
        \thline
        FID          & 2.75 & 2.53 & \textbf{2.48} & \textbf{2.48} & 2.50\\
        IS           & 9.58 & 9.64 & \textbf{9.70} & 9.61 & 9.58\\
        \boldtoprule
    \end{tabular}
    }
    \vspace{-1em}
    \label{tab:lambda}
\end{wraptable} 
\subsubsection{The Impact of the Choices of $\lambda$ on the Performance of QC-NCSN++}
\label{apx:experiments:lambda}
Based on our preliminary results on the toy environment, we perform a hyperparameter sweep for $\lambda=$\{1e-3, 5e-4, 1e-4, 5e-5\}, and report the best results on the real-world experiments. Table~\ref{tab:lambda} presents the evaluation results of FID and IS under different choices of $\lambda$. In this experiment, the PC sampler is adopted and NFE is fixed at 1,000. The experimental results presented on the rows `FID' and `IS' demonstrate that QC-NCSN++ achieves its best sampling performance when $\lambda$ is selected as 0.0001. Based on this finding, we choose $\lambda$ to equal to 0.0001 throughout the experiments in Section~\ref{sec:experimental_results}.

\subsubsection{Visualized Examples}
\label{apx:experiments:uncurated_examples}
Fig.~\ref{fig:figure} depict a few uncurated visualized examples that qualitatively demonstrate the sampling quality of QC-NCSN++ on the real-world datasets.
\begin{figure}[h]
    \centering
    \vspace{1em}
    \includegraphics[width=0.65\linewidth]{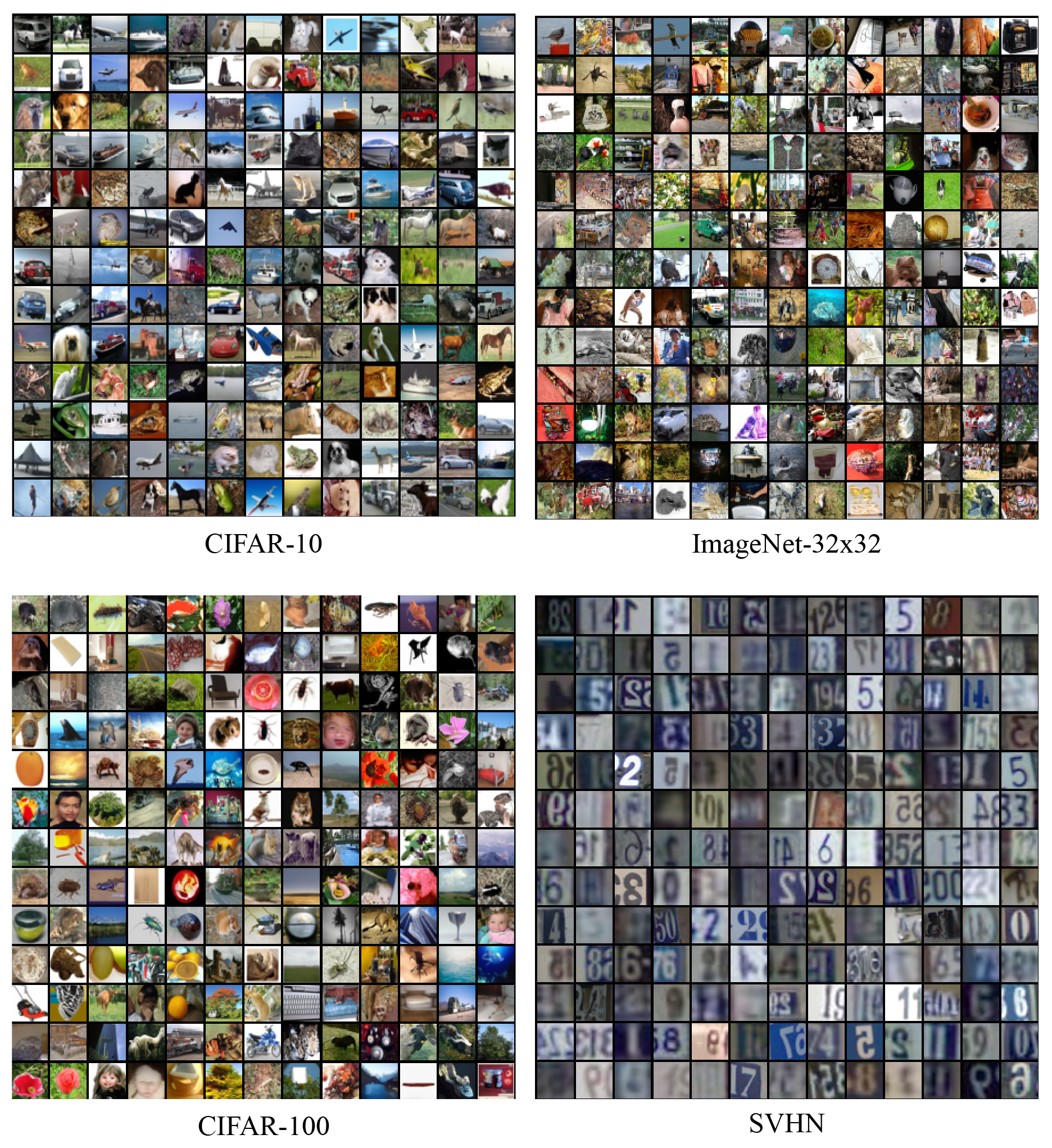}
    \vspace{-1em}
    \caption{A number of visualized examples generated using QC-NCSN++.}
    \label{fig:figure}
\end{figure}

\begin{figure}[h]
    \centering
    \vspace{-1em}
    \includegraphics[width=\linewidth]{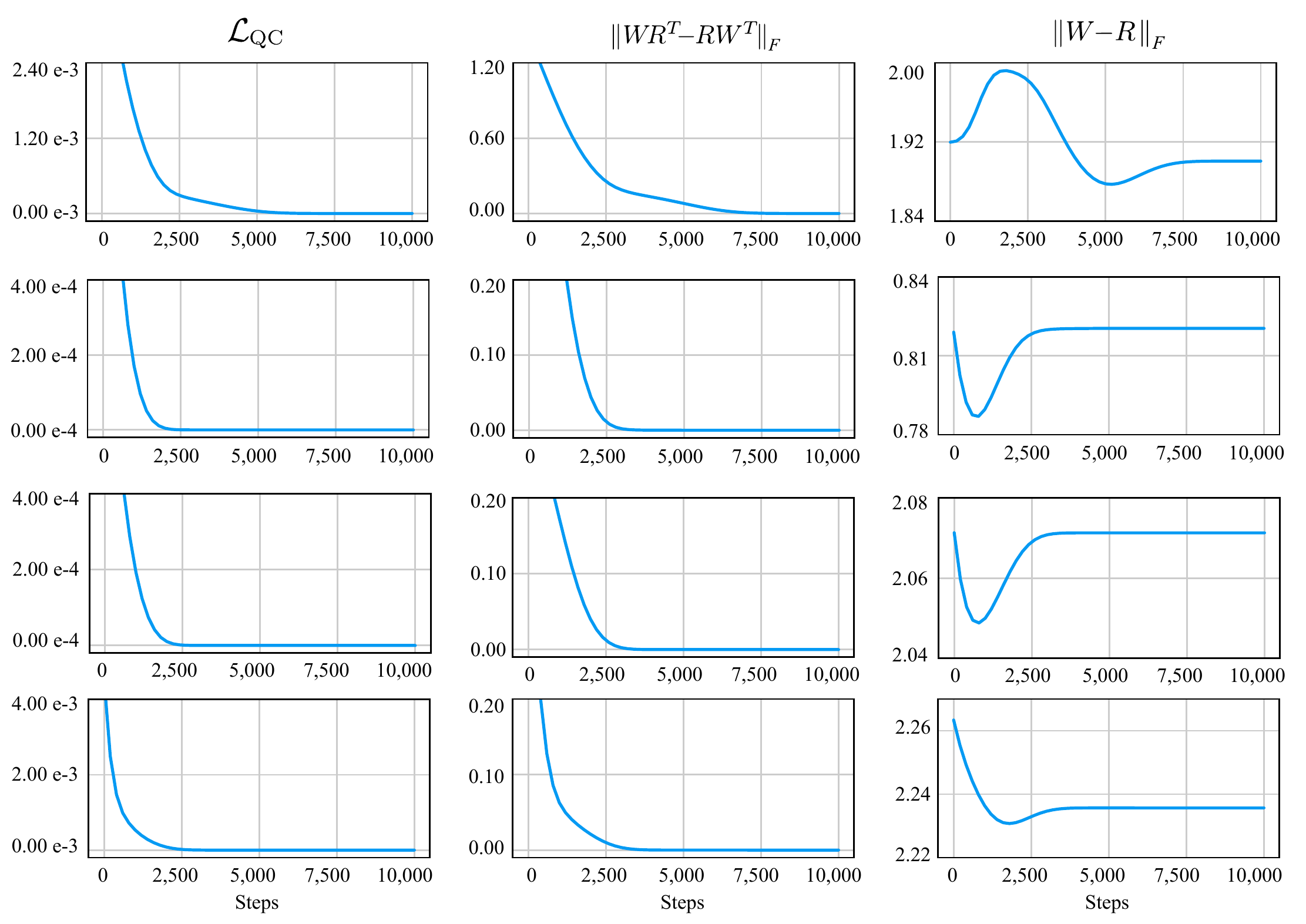}
    \vspace{-2em}
    \caption{The trends of $\forbnorm{\weightw\weightr^T-\weightr\weightw^T}$ and $\forbnorm{\weightw-\weightr}$ during the minimization process of $\reg$. The `steps' on the x-axes refer to the training steps.}
    \label{fig:sym_exp+}
\end{figure}


\end{document}